\documentclass[twoside]{article}
\usepackage[accepted]{aistats2020}

\usepackage[utf8]{inputenc} % allow utf-8 input
\usepackage[T1]{fontenc}    % use 8-bit T1 fonts
\usepackage{hyperref}       % hyperlinks
\usepackage{url}            % simple URL typesetting
\usepackage{booktabs}       % professional-quality tables
\usepackage{amsfonts}       % blackboard math symbols
\usepackage{nicefrac}       % compact symbols for 1/2, .
\usepackage{microtype}      % microtypography
\usepackage{natbib}
\usepackage{setspace}

\usepackage{times}
\usepackage[normalem]{ulem}
\usepackage{graphicx}
\usepackage{caption}
\usepackage{amssymb}

\usepackage{amsmath}
\usepackage{amsfonts}
\usepackage[english]{babel}
\usepackage[bottom]{footmisc}
\usepackage{wrapfig}
\usepackage[noend]{algpseudocode}
\usepackage{algorithm}
\usepackage{algpseudocode}
\usepackage{amsthm}
\usepackage{siunitx}
\usepackage{array}
\usepackage{multirow}

\usepackage{amsmath}
\DeclareMathOperator\erf{erf}
\DeclareMathOperator*{\argmax}{arg\,max}
\DeclareMathOperator*{\argmin}{arg\,min}
\usepackage{subfig}
\usepackage{adjustbox}

\usepackage[colorinlistoftodos,prependcaption,textsize=tiny]{todonotes}
\algnewcommand{\IIf}[1]{\State\algorithmicif\ #1\ \algorithmicthen}
\algnewcommand{\EElse}[1]{\State \algorithmicelse\ #1\ }
\algnewcommand{\EndIIf}{\unskip}

\newcommand{\squishlist}{
   \begin{list}{$\bullet$}
    { \setlength{\itemsep}{2pt}    \setlength{\parsep}{0pt}
      \setlength{\topsep}{5pt}     \setlength{\partopsep}{0pt}
      \setlength{\leftmargin}{1.35em} \setlength{\labelwidth}{1em}
      \setlength{\labelsep}{0.5em} } }

\newcommand{\squishend}{
    \end{list}  }

\newtheorem{mydef}{Definition}

\newtheorem{theorem}{Theorem}

\newtheorem{lemma}{Lemma}

\newtheorem{proposition}{Proposition}

\newcommand{\covvec}{\mathbf{r}}

\newcommand{\piInf}[1]{\pi_{\min}(#1)}
\newcommand{\piSup}[1]{\pi_{\max}(#1)}
\newcommand{\piInfL}[1]{\pi_{\min}^L(#1)}
\newcommand{\piInfU}[1]{\pi_{\min}^U(#1)}
\newcommand{\piSupL}[1]{\pi_{\max}^L(#1)}
\newcommand{\piSupU}[1]{\pi_{\max}^U(#1)}

\newcolumntype{R}[2]{%
    >{\adjustbox{angle=#1,lap=\width-(#2)}\bgroup}%
    c%
    <{\egroup}%
}

\begin{document}

\twocolumn[
\aistatstitle{Adversarial Robustness Guarantees for Classification with Gaussian Processes}

% The \author macro works with any number of authors. There are two commands
% used to separate the names and addresses of multiple authors: \And and \AND.
%
% Using \And between authors leaves it to LaTeX to determine where to break the
% lines. Using \AND forces a line break at that point. So, if LaTeX puts 3 of 4
% authors names on the first line, and the last on the second line, try using
% \AND instead of \And before the third author name.
\aistatsauthor{ Arno Blaas$^*$ \And Andrea Patane$^*$ \And  Luca Laurenti$^*$}
\aistatsaddress{ Department of Engineering Science \\ University of Oxford \And Department of Computer Science \\ University of Oxford \And Department of Computer Science \\ University of Oxford}
\aistatsauthor{  Luca Cardelli \And Marta Kwiatkowska \And Stephen Roberts}
\aistatsaddress{Department of Computer Science \\ University of Oxford \And Department of Computer Science \\ University of Oxford \And Department of Engineering Science \\ University of Oxford}
]

%syntax highlighting
\newcommand{\varr}[1]{{\color{green!60!black}#1}}
\newcommand{\varc}[1]{{\color{blue!60!black}#1}}
\newcommand{\varl}[1]{{\color{red!60!black}#1}}

%pre-post constraint
\mathchardef\mhyphen="2D
\newcommand{\prepost}[1]{\mathsf{pre\mhyphen post}_{#1}}
\newcommand{\inv}[1]{\mathsf{inv}_{#1}}

\newcommand{\ap}[1]{{\color{purple}[AP: #1]}}
\newcommand{\LL}[1]{{\color{red}[LL: #1]}}
\newcommand{\AB}[1]{{\color{green}[AB: #1]}}
\newcommand{\MK}[1]{{\color{blue}[MK: #1]}}
\newcommand{\SR}[1]{{\color{cyan}[SR: #1]}}
\newcommand{\rev}[1]{{\color{blue}#1}}
\newcommand\blfootnote[1]{%
  \begingroup
  \renewcommand\thefootnote{}\footnote{#1}%
  \addtocounter{footnote}{-1}%
  \endgroup
}
%\begin{center} \Large{\bf Syntax-Guided Optimal Synthesis for Chemical Reaction Networks} \end{center}

\begin{abstract}
%We consider Bayesian classification with  Gaussian process classifiers (GPCs) and define \rev{adversarial} robustness of a classifier in terms  of the worst-case difference in the classification probabilities with  respect to input perturbations. For a compact subset of the input space  $T\subseteq \mathbb{R}^d$ evaluating this robustness reduces to computing the minimum and maximum of the classification probabilities for all points  in $T$, \rev{which in general cannot be done in finite time}. 
%However, using the theory of Gaussian processes, we develop a framework that, for a given freely selectable error threshold  $\epsilon>0$, computes $\epsilon-$exact approximations to these minimum and maximum values in finite time by introducing suitable lower and upper bounds to the problem and iteratively refining them inside a taylored branch and bound algorithm until convergence.
%In experiments on a synthetic dataset and two real-world datasets, we demonstrate the usefulness of our results in tasks such as verifying safety, comparing adversarial robustness among different training regimes, or even conducting interpretability analyses for GPCs.

We investigate adversarial robustness of Gaussian Process Classification (GPC) models. Given a compact subset of the input space $T\subseteq \mathbb{R}^d$ enclosing a test point $x^*$ and a GPC trained on a dataset $\mathcal{D}$, we aim to compute the minimum and the maximum classification probability for the GPC over all the points in $T$.
In order to do so, we show how functions lower- and upper-bounding the GPC output in $T$ can be derived, and implement those in a branch and bound optimisation algorithm. For any error threshold $\epsilon > 0$ selected \emph{a priori}, we show that our algorithm is guaranteed to reach values $\epsilon$-close to the actual values in finitely many iterations.
We apply our method to investigate the robustness of GPC models on a 2D synthetic dataset, the SPAM dataset and a subset of the MNIST dataset, providing comparisons of different GPC training techniques, and show how our method can be used for interpretability analysis. 
Our empirical analysis suggests that GPC robustness increases with more accurate posterior estimation.

\end{abstract}

\section{INTRODUCTION}\blfootnote{$^*$Equal Contributions.}
Adversarial examples (i.e.\ input points intentionally crafted to trick a model into misclassification) have raised serious concerns about the security and robustness of models learned from data \citep{biggio2018wild}.
Since test accuracy fails to account for the behaviour of a model in adversarial settings, the development of techniques capable of quantifying the adversarial robustness of machine learning models is an essential pre-condition for their application in safety-critical scenarios \citep{ribeiro2016should}.
In particular, Gaussian Processes (GPs), thanks to their favourable analytical properties, allow for the computation of the uncertainty over model predictions in Bayesian settings, which can then be propagated through the decision pipeline to facilitate decision-making \citep{rasmussen2004gaussian}.
However, while techniques for the computation of robustness guarantees have been developed for a variety of non-Bayesian machine learning models \citep{katz2017reluplex,huang2017safety,biggio2018wild}, to the best of our knowledge studies of \textit{adversarial classification robustness} of GPs have been limited to statistical (i.e.\ input distribution dependent) \citep{abdelaziz2017data} and heuristic analyses \citep{grosse2018limitations,bradshaw2017adversarial}, and methods for the computation of adversarial robustness guarantees are missing. %have been little studied.

In this work, given a trained GP Classification (GPC) model and a compact subset of the input space $T\subseteq \mathbb{R}^d$, we pose the problem of computing the maximum and minimum of the GPC class probabilities over all $x \in T$.
We show that such values naturally allow us to compute robustness properties %widely 
employed for analysis of deep learning models \citep{ruan2018reachability}, e.g.\ can be used to provide guarantees of non-existence of adversarial examples
%against the existence of adversarial perturbations 
and for the computation of classification ranges for sets of input points.
Unfortunately,  exact direct computation of the maximum and minimum class probabilities over compact sets is not possible, as these would require providing an \textit{exact solution} of a global non-linear optimisation problem, for which no general method exists \citep{neumaier2004complete}. 
We show how upper and lower bounds for the maximum and minimum classification probabilities of GPCs can be computed on any given compact set $T$, and then iteratively refine these bounds in a branch and bound algorithmic scheme until convergence to the minimum and maximum is obtained.

Specifically, through discretisation of the GPC latent space, we derive an upper and lower bound on the GPC class confidence output by analytically optimising a set of Gaussian integrals, whose parameters depend upon extrema of the GPC posterior mean and variance in $T$.
We show how the latter can be bounded by solving a set of convex quadratic and linear programming problems, for which solvers are readily available \citep{boyd2004convex}. 
Finally, for any given error tolerance 
$\epsilon > 0,$ we prove that there exists a discretisation of the latent space that ensures convergence of the branch and bound to values $\epsilon$-close to the actual maximum and minimum class probabilities in finitely many steps.
The method we propose
%proposed here 
is anytime (the bounds provided are at every step an over-estimation of the actual classification ranges over $T$, and can hence be used to provide guarantees) and $\epsilon$-exact (the actual values are retrieved in finitely many steps up to an error $\epsilon$ selected a-priori).

We apply our approach to analyse the robustness profile of GPCs on 
a two-dimensional dataset, the SPAM dataset, and a feature-based analysis of a binary and a $3$-class subset of the MNIST dataset. 
In particular, we compare the guarantees computed by our method with the robustness estimation approximated by adversarial attack methods for GPCs  \citep{grosse2018limitations}, discussing in which settings the latter fails. Then, we analyse the effect of approximate Bayesian inference techniques and hyper-parameter optimisation procedures on the GPC adversarial robustness.
Interestingly, across the three datasets analysed here, we observe that approximation based on Expectation Propagation \citep{minka2001expectation} gives more robust classification models than Laplace approximation \citep{rasmussen2004gaussian}, and that 
GPC robustness increases with the number of training epochs.
%, as opposed to what happens in deep neural networks, where a marked trade-off between robustness and accuracy is generally observed \citep{cardelli2019statistical,tsipras2018robustness}.
Finally, we show how robustness can be used to perform interpretability analysis of GPC predictions and compare our methodology with LIME \citep{ribeiro2016should}.

In summary, the paper presents the following contributions:
\begin{itemize}
    \item We develop a method for computing lower and upper bounds for GPC probabilities over compact sets.
    \item We incorporate the bounding procedure in a branch and bound algorithm, which we show to converge for any specified error  $\epsilon >0$ in finitely many steps. 
    \item We empirically evaluate the robustness of a variety of GPC models on three datasets, and 
    %further 
    demonstrate how our method can be used for interpretability analysis.
\end{itemize}

\paragraph{Related Work}
%\ap{here}
%\ap{Consider extending this a bit and go a bit more in details with differences with [20]. It is not just a matter of regression vs. Classification. Here we take into accoutn decision making, and guarantee convergence, and have anytime bounds.}

Different notions of robustness have been studied for GPs. For instance, \citet{kim2008outlier} consider robustness against outliers, while \citet{hernandez2011robust} study robustness against labelling errors.
In this paper  we consider robustness against local adversarial perturbations, whose quantification for Bayesian models is a problem addressed in several papers. Heuristic approaches based on studying adversarial examples  are developed by \citet{grosse2018limitations,feinman2017detecting}.
Formal guarantees are derived by \citet{cardelli2018robustness,bogunovic2018adversarially,smith2019adversarial} for GPs and by \citet{cardelli2019statistical} for Bayesian neural networks. 
In particular, \citet{cardelli2018robustness}  derive an upper bound on the probability that there
exists a point in the neighbourhood of a given test point of a GP such that the prediction
of the GP on the latter differs from the initial test input point
by at least a specified threshold, whereas  \citet{bogunovic2018adversarially} consider a GP optimisation algorithm  in which the returned solution is guaranteed to be robust to adversarial perturbations with a certain probability. 
The problem and the techniques developed in this paper are substantially different from both of these.
First, we consider a classification problem,
for which the bounds in the referenced papers cannot be applied due to its non-Gaussian nature. 
Then, the approach in this paper gives stronger (i.e., non-probabilistic) guarantees, is  guaranteed to converge to  any given error $\epsilon >0$ in finite time, and is anytime (i.e., at any time it gives sound  upper and lower bounds of the classification probabilities).
This also differs from \citet{cardelli2019statistical}, where the authors consider statistical guarantees that require the solution of many non-linear optimisation problems (one for each sample from the posterior distribution).
Our approach also differs from that in \citet{smith2019adversarial},  where the authors give guarantees for GPC in a binary classification setting under the $L_0$-norm and  only consider the mean of the distribution in the latent space without taking into account the uncertainty intrinsic in the GPC framework. In contrast, our approach also considers multi-class classification, takes into account the full posterior distribution and allows for exact (up to $\epsilon>0$) computation under any $L_p$-norm.

\section{BAYESIAN CLASSIFICATION WITH GAUSSIAN PROCESSES}\label{Section:ClassificationIntro}
In this section we provide background for classification with GP priors. %, and briefly review the approximate inference methods employed in the remainder of the paper. 
We consider the classification problem associated to a dataset $\mathcal{D}=\{(x,y) \, | \, x \in \mathbb{R}^d,y \in \{1,...,C \} \}$.
In GPC settings, given a test point $x^* \in \mathbb{R}^d$, the probability assigned by the GPC to $x^*$ belonging to class $c$ is given by:
\begin{equation}
\label{Eq:MultiClassClassification}
    \pi^c(x^*|\mathcal{D}) = \int \sigma^c(\bar{f}) p(f(x^*)=\bar{f} | \mathcal{D}) d\bar f,
\end{equation}
where $f(x^*)=[f^1(x^*), \ldots ,f^C(x^*)]$ is the latent function vector, $\sigma^c:\mathbb{R}^C \to[0,1]$ is the likelihood function for class $c$, $p(f(x^*)=\bar{f} | \mathcal{D})$ is the predictive posterior distribution of the GP, and the integral is computed over the $C$-dimensional latent space \citep{rasmussen2004gaussian}.
The vector of class probabilities, ${\Pi}(x^*)=[\pi^1(x^* | \mathcal{D}),...,\pi^C(x^*| \mathcal{D})]$, can be computed by iterating Eqn \eqref{Eq:MultiClassClassification} for each class $c=1,\ldots,C$.
Of particular interest in applications is the binary classification case (i.e., when C=2), which
%In fact, setting $C=2$ 
leads to a significant simplification of the inference equations and techniques, while still encompassing important practical applications \citep{nickisch2008approximations} (notice that the binary case can be used for multi-class classification as well by means of, e.g., one-vs-all classifiers \citep{rasmussen2004gaussian,hsu2002comparison}).
More specifically, in this case it suffices to compute $\pi (x^* | \mathcal{D})= \int \sigma (\bar{f}) p(f(x^*)=\bar{f} | \mathcal{D})d\bar{f} := \pi^1(x^* | \mathcal{D}) $, with $f$ being a univariate latent function and setting  $\pi^2(x^* | \mathcal{D}) := 1 - \pi (x^* | \mathcal{D})$.
%More specifically, in this case, if we let $\pi (x^* | \mathcal{D})$ denote the probability of $x^*$ belonging to class $1$, then the probability of $x^*$ belonging to class $2$ is simply $1-\pi (x^*|\mathcal{D})$.
%Therefore, it suffices to compute $\pi (x^* | \mathcal{D})= \int \sigma (\bar{f}) p(f(x^*)=\bar{f} | \mathcal{D})d\bar{f},$ with $f$ being a univariate latent function.
In other words, when $C=2$ the latent space of the GPC model is one-dimensional. 
%This simplification can also be used for the multiclass case ($C > 2$) by reducing it to the binary case via one-vs.-all classification \citep{rasmussen2004gaussian},\citep{hsu2002comparison}.

Unfortunately, even under the GP prior assumption, the posterior distribution in classification settings, $p(f(x^*)=\bar{f} | \mathcal{D})$, is non-Gaussian and intractable \citep{rasmussen2004gaussian}.
Several approximation methods have been developed to perform GPC inference, either by sampling (e.g.\ Markov Chain Monte Carlo), or by developing suitable analytic approximations of the posterior distribution.
In this work we focus on Gaussian analytic approximations, that is, we employ GPC methods that perform approximate inference of  $p(f(x^*)=\bar{f} | \mathcal{D})$ by estimating a Gaussian distribution $q(f(x^*)=\bar{f} | \mathcal{D})  = \mathcal{N}(\bar f \, | \, \mu(x^*),\Sigma(x^*)  ) $. 
The latter is then used at inference time in Eqn \eqref{Eq:MultiClassClassification} in place of the exact posterior $p(f(x^*)=\bar{f} | \mathcal{D})$\footnote{With an abuse of notation, in the rest of the paper we will consider $\pi^c(x^*|\mathcal{D}) = \int \sigma^c(\bar{f}) q(f(x^*)=\bar{f} | \mathcal{D}) d\bar f$.}.
In particular, in Section \ref{Sec:ExperimentalResults} we will give experimental results for when $q$ is derived using either \textit{Laplace approximations} \citep{williams1998bayesian} or \emph{Expectation Propagation} (EP) \citep{minka2001expectation}.
However, we remark that the methods presented in this paper do not depend on the particular Gaussian approximation method used
and can be trivially extended to the case where $q$ is a mixture of Gaussian distributions.

\section{ADVERSARIAL ROBUSTNESS}
\label{sec:prob_form}
Given a GPC model trained on a dataset $\mathcal{D}$ and a test point $x^*,$ we are interested in quantifying the adversarial robustness of the GPC in a neighborhood of $x^*$.
To do so, for a compact set $T$ and a class $c \in \{1, \ldots,C \}$, we pose the problem of computing 
the minimum and the maximum that the GPC assigns to the probability of class $c$ in $T$, that is:
%\begin{equation}\label{Eq:infSubROsustness} 
%    \begin{aligned}
%   \pi_{\min}^c(T) &:=\min_{x \in T} \pi^c (x \, | \, \mathcal{D})  \\
%     \pi_{\max}^c(T) &:=\max_{x \in T} \pi^c (x \, | \, \mathcal{D}),
%    \end{aligned}
%\end{equation}
\begin{align}
   \label{Eq:infSubROsustness} 
   \hspace*{-0.1cm} \pi_{\min}^c(T) \hspace*{-0.05cm} := \hspace*{-0.05cm} \min_{x \in T} \pi^c (x | \mathcal{D}) \quad 
     \pi_{\max}^c(T) \hspace*{-0.05cm} :=\hspace*{-0.05cm} \max_{x \in T} \pi^c (x  |  \mathcal{D})
\end{align}
%\begin{align}
%   \piInf{T} &:=\min_{x \in T} \pi^c (x \, | \, \mathcal{D})  \label{Eq:infSubROsustness}  \\  \piSup{T} &:=\max_{x \in T} \pi^c (x \, | \, \mathcal{D}), \nonumber
%\end{align}
%where we omit $c$ in the subscript of the LHS for notational brevity.  %\MK{I think this is confusing - I would keep the subscript in Equation 2 and then say that we simplify as we work with binary and drop it assuming c=1}
The computation of the classification %\MK{you mean extrema not ranges - range is the same as interval} 
extrema in $T$ allows us to determine the reachable interval of class probabilities over $T$.
In the case in which $T$ is defined as a neighborhood around a test point $x^*$, Eqn \eqref{Eq:infSubROsustness} provides a quantification of the local GPC robustness at $x^*$, that is, %of the GPC robustness 
against local adversarial perturbations.
Unfortunately, exact computation of Eqn \eqref{Eq:infSubROsustness} involves the solution of two non-linear optimisation problems, for which no general  solution method exists. 
%Furthermore, in general, the inference equation for $ \pi^c(x^*|\mathcal{D})$ cannot be written in closed form.
Nevertheless, in Section \ref{Sec:BinaryBounds} we derive a branch and bound scheme for the anytime computation of the classification ranges of Eqn \eqref{Eq:infSubROsustness} that is guaranteed to converge in finitely many iterations up to any arbitrary error tolerance $\epsilon > 0 $.

In what remains of this section we discuss two notions of adversarial robustness %\MK{not widely}
%widely 
employed for the  analysis of deep learning models \citep{ruan2018reachability} that arise as particular instances of Eqn \eqref{Eq:infSubROsustness}, which will be investigated in the experimental results discussed in Section \ref{Sec:ExperimentalResults}.
%In fact, computing the classification ranges over $T$ allow us to check whether these two robustness notions are satisfied or not. 
%\MK{You don't say how to check}
\begin{mydef}{(Adversarial Local Robustness)}\label{ProbDef:Robustness}
%Consider the training dataset $\mathcal{D}$. 
Let $T\subseteq \mathbb{R}^d$ and $x^* \in T.$ Then, for $\delta >0$ we say that the classification of $x^*$ is $\delta-$robust in $T$ iff $
    \forall x \in T, \, |\Pi(x^*)-\Pi(x)|\leq \delta,
$
where $|\cdot|$ is a given norm.
\end{mydef}
If $T$ is a $\gamma-$ball around a test point $x^*$, then robustness defined in Definition \ref{ProbDef:Robustness} allows one to quantify how much, in the worst case, the prediction in $x^*$  can be affected by input perturbations of radius no greater than $\gamma$.
Adversarial examples are defined in terms of invariance of the classification in $T$ w.r.t.\ the label of a test point $x^*$. For the case of the Bayesian optimal classifier, this is defined as follows.
\begin{mydef}{(Adversarial Local Safety)}\label{ProbDef:Safety}
%Consider the training dataset $\mathcal{D}$. 
Let $T\subseteq \mathbb{R}^d$ and $x^* \in T$. Then, we say that the classification of $x^*$ is safe in $T$ iff $\forall x \in T, \, \argmax_{c \in \{1,...,C \}}\pi^c(x \vert \mathcal{D}) = \argmax_{c \in \{1,...,C \}}\pi^c(x^* \vert \mathcal{D} ).$
\end{mydef}
Adversarial local safety establishes whether adversarial examples exist in $T$, yielding formal guarantees against adversarial attacks for GPCs. %, such as those developed in \cite{grosse2018limitations}.
If we again consider $T$ to be a $\gamma-$ball around $x^*$, the satisfaction of Definition \ref{ProbDef:Safety} guarantees that it is not possible to cause a misclassification by perturbing $x^*$ by a magnitude of up to $\gamma$\footnote{In the multiclass case to check if $x^*$ is safe in $T$, for $\bar{c}=\argmax_{c \in \{1,...,C \}} \pi^c(x^*| \mathcal{D}),$ we need to check that $\min_{x\in T} \left( \pi^{\bar{c}}(x | \mathcal{D}) - \max_{c\neq \bar{c}}\pi^{
{c}}(x | \mathcal{D})  \right)>0$, which can be computed with a trivial extension of the results presented in this paper.}.
\section{BOUNDS FOR BINARY CLASSIFICATION}
\label{Sec:BinaryBounds}
In this section we show how the classification ranges of a two-class GPC model in any given compact set $T \subseteq \mathbb{R}^d$ can be computed up to any arbitrary precision $\epsilon > 0 $. As explained in Section \ref{Section:ClassificationIntro}, the latent space of the GPC model is one-dimensional in this case, and we thus omit the class superscript $c$ in this section. 
The extension to the multi-class scenario is then described in Section \ref{Sec:Multiclass}.
Proofs for the results %Propositions and Theorems 
stated %in this section 
are given in the Supplementary Material. 

\textbf{Outline of Approach }
\begin{figure*}
	\centering
	 \includegraphics[width = 0.76\textwidth]{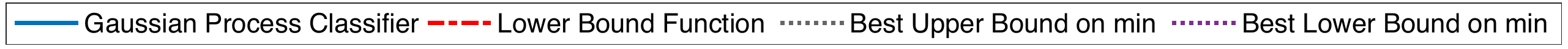} \\
	\includegraphics[width = 0.36\textwidth]{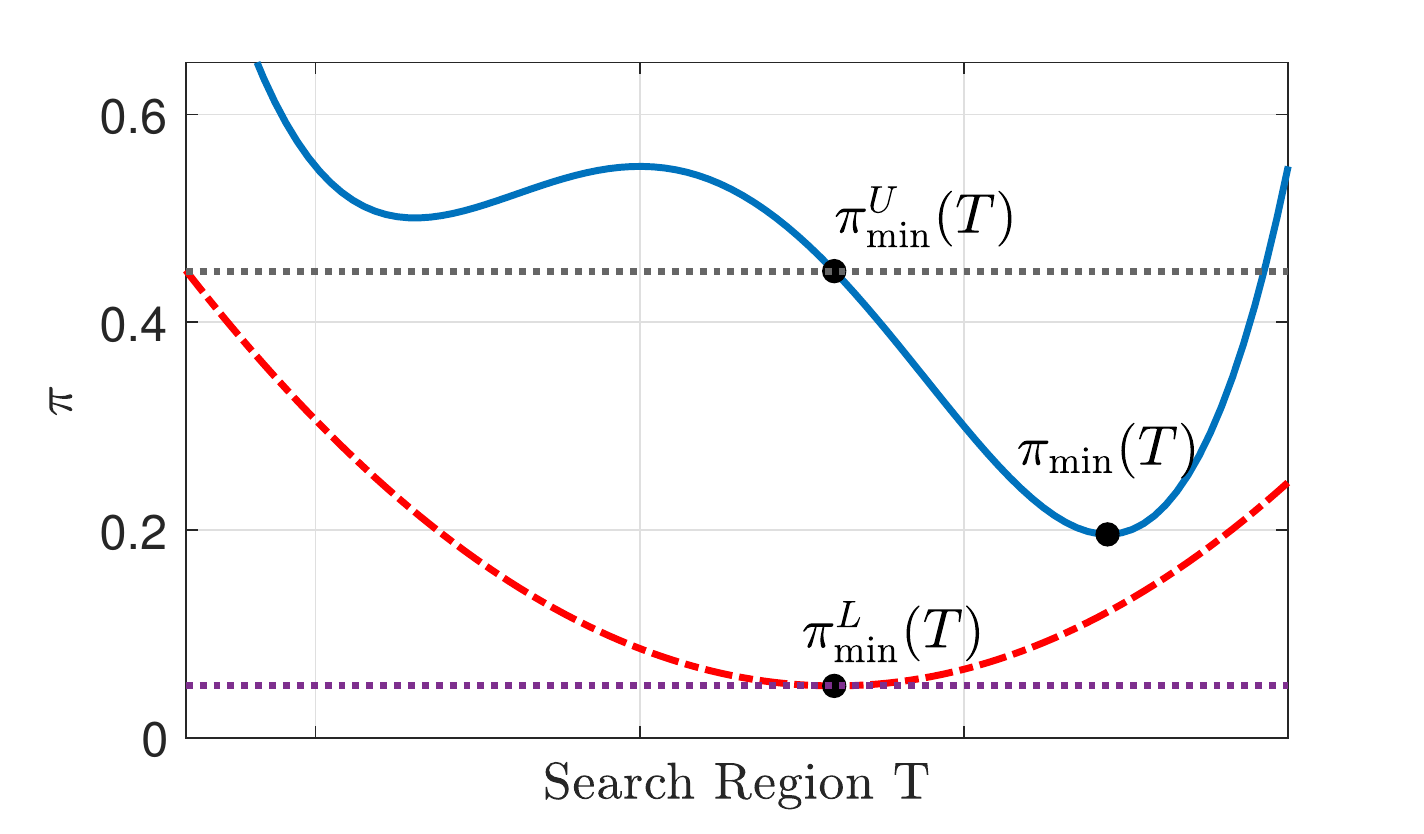}  
	\includegraphics[width = 0.36\textwidth]{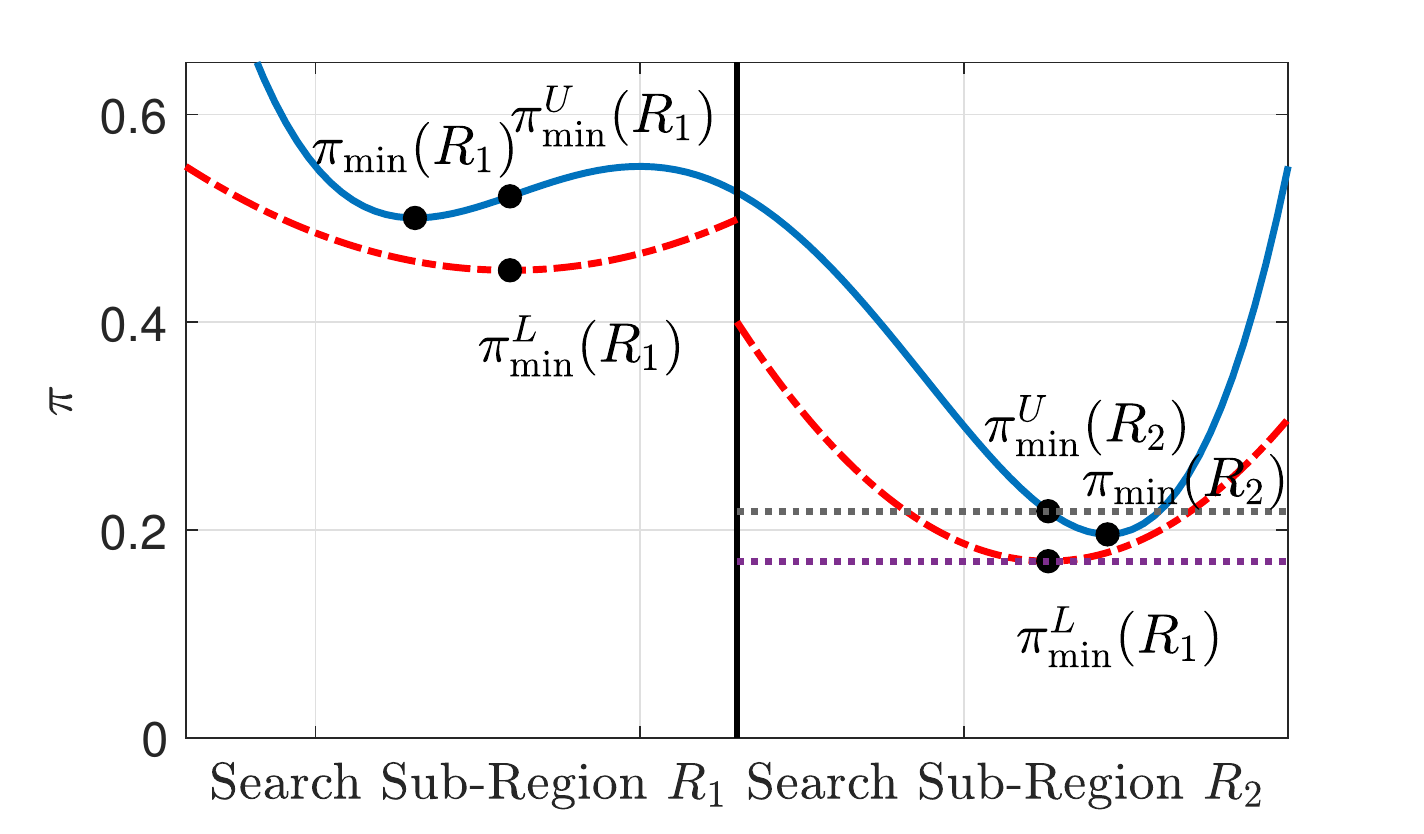}
	\caption{\textbf{Left:} Computation of upper and lower bounds on $\piInf{T}$, i.e. the minimum  of the classification range on the search region $T$. \textbf{Right:} The search region is repeatedly partitioned into sub-regions according to Algorithm \ref{alg:bnb_sketch} (only first partitioning visualised), reducing the gap between best lower and upper bounds until convergence (up to $\epsilon$) is reached. }
	\label{fig:branch_and_bound}
\end{figure*}
An outline of our approach is depicted in Figure \ref{fig:branch_and_bound} for the computation of $\piInf{T}$ %\footnote{In the case of binary classification, as explained in Section \ref{Section:ClassificationIntro}, we can omit the class superscript $c.$}
over a one-dimensional set $T$ plotted along the x-axis (the method for the computation of $\piSup{T}$ is analogous).
For any given region $T$ we aim to compute lower and upper bounds on both $\piInf{T}$ and $\piSup{T}$, that is, we compute real values $\piInfL{T}$, $\piInfU{T}$, $\piSupL{T}$ and $\piSupU{T}$ such that:
\begin{align}\label{eq:lower_upper_condition1}
    \piInfL{T} \leq &\piInf{T} \leq \piInfU{T} \\ 
 \label{eq:lower_upper_condition2}   \piSupL{T} \leq &\piSup{T} \leq \piSupU{T}.
\end{align}
In order to do so, we compute a lower and an upper bound function (the lower bound function is depicted with a dashed red curve in Figure \ref{fig:branch_and_bound}) to the GPC output (solid blue curve) in the region $T$. % by applying Proposition \ref{Theorem:BOundsGenericBIclass}.
We then find the minimum of the lower bound function, $\piInfL{T}$ (shown in the plot), and the maximum of the upper bound function, $\piSupU{T}$ (not shown).
%By definition of lower and upper bounds, $\piInfL{T}$ and $\piSupU{T}$ thus computed satisfy Eqns \eqref{eq:lower_upper_condition1} and  \eqref{eq:lower_upper_condition2}, respectively. 
Then, valid values for $\piInfU{T}$ and $\piSupL{T}$ can be computed by evaluating the GPC on any point in $T$ (a specific $\piInfU{T}$ is depicted in Figure \ref{fig:branch_and_bound}).
%We describe in the Supplementary Material how these can be computed in practice in order to speed up convergence.
Finally, we iteratively refine the lower and upper bounds computed in $T$ with a branch and bound algorithm.
Namely, the region $T$ is recursively subdivided into sub-regions, for which we compute new (tighter) bounds, until these converge up to a desired tolerance $\epsilon > 0$.

\textbf{Computation of Bounds}
In this paragraph we show how to compute $\piSupU{T}$, an upper bound on the maximum, and $\piInfL{T}$,  a lower bound on the minimum of the GPC outputs. 
We work on the assumption that the likelihood function $\sigma(f)$ is a monotonic, non-decreasing, and continuous function of the latent variable (notice that this is satisfied by commonly used likelihood functions, e.g., logistic and probit  \citep{kim2006bayesian}).
In the following proposition we show how the GPC output can be upper- and lower-bounded in $T$ by a summation of Gaussian integrals.

%For the binary case we assume that the likelihood function $\sigma(f)$ is a monotonic, non-decreasing, and continuous function of the latent variable, assumptions that are satisfied by all commonly used likelihood functions (e.g.\ logistic and probit \cite{kim2006bayesian}). 
%We can then derive Proposition \ref{Theorem:BOundsGenericBIclass}, which guarantees that $\piInf{T}$ and $\piSup{T}$ can be bounded by a summation of Gaussian integrals.
%Its proof, as well as the proof of all other propositions and theorems in the paper, can be found in Appendix.

\begin{proposition}
\label{Theorem:BOundsGenericBIclass}% \ap{Partizione ad intervalli}
Let $\mathcal{S}=\{S_i \mid i \in \{1,...N \}\}$ be a partition of $\mathbb{R}$ (the latent space) in a finite set of intervals. Call $a_i=\inf_{\bar f \in S_i} \bar f$ and $b_i=\sup_{\bar f \in S_i} \bar f$. Then,  it holds that:
%\begin{align}\label{Eq:LowerBOundGenericBidimensional}
%    \sum_{i=1}^{N}  \sigma(a_i) \inf_{x \in T} \int_{a_i}^{b_i} \mathcal{N}(\bar f | \mu(x),\Sigma(x))d \bar{f}\leq   \pi(x | \mathcal{D})  \leq \sum_{i=1}^{N}  \sigma(b_i) \sup_{x \in T} \int_{a_i}^{b_i} \mathcal{N}(\bar f | \mu(x),\Sigma(x))d \bar{f}
%\end{align}
\begin{align}
    \label{Eq:LowerBOundGenericBidimensional}
    & \piInf{T} \geq  \sum_{i=1}^{N}  \sigma(a_i) \min_{x \in T} \int_{a_i}^{b_i} \mathcal{N}(\bar f | \mu(x),\Sigma(x))d \bar{f}
    \\
    \label{Eq:UpperBOundGenericBidimensional}
    & \piSup{T} \leq  \sum_{i=1}^{N}  \sigma(b_i) \max_{x \in T} \int_{a_i}^{b_i} \mathcal{N}(\bar f | \mu(x),\Sigma(x))d \bar{f},
\end{align}
where $\mu(x)$ and $\Sigma(x)$ are mean and variance of the predictive posterior $q(f(x)=\bar f | \mathcal{D})$.
\end{proposition}
Proposition \ref{Theorem:BOundsGenericBIclass} guarantees that the GPC output in $T$ can be bounded by solving $N$ optimisation problems. Each of these problems seeks to find the mean and variance that maximise or minimise the integral of a Gaussian over $T$.
This has been studied by \cite{cauchi2019efficiency} for variance-independent points and is generalised in the following proposition. 
We introduce the following notation for lower and upper bounds on mean and variance in $T$:
\begin{align} \label{eq:mean_variance_bounds1}
\mu^L_{T} &\leq \min_{x \in T} \mu(x) \quad \; \mu^U_T\geq \max_{x \in T}\mu(x)\quad \\ 
\label{eq:mean_variance_bounds2} \Sigma^L_{T} &\leq \min_{x \in T} \Sigma(x) \quad \Sigma^U_T \geq \max_{x \in T}\Sigma(x),
\end{align}
Then by inspection of the derivatives of the integrals in Eqns \eqref{Eq:LowerBOundGenericBidimensional} and \eqref{Eq:UpperBOundGenericBidimensional} the following proposition follows.
\begin{proposition}\label{Prop:Gaussian}
Let $\mu^m = \frac{a+b}{2}$ and $\Sigma^m(\mu) = \frac{{(\mu-a)^2-(\mu-b)^2}}{{2 \log \frac{\mu-a}{\mu-b}}} $. Then it holds that:
\begin{align}
    &\max_{x \in T}\int_{a}^{b} \mathcal{N}(\bar f | \mu(x),\Sigma(x))d \bar{f} \leq \int_{a}^{b} \mathcal{N}( \bar {f} \vert \overline{\mu} , \overline{\Sigma} ) d \bar{f} \nonumber \\
    &= \frac{1}{2} \left( \erf \left(\frac{\overline{\mu}-a}{\sqrt{2\overline{\Sigma}}}\right) - \erf \left(\frac{\overline{\mu}-b}{\sqrt{2\overline{\Sigma}}}\right)  \right) \label{eq:gen_like2}\\
    &\min_{x \in T}\int_{a}^{b} \mathcal{N}(\bar f | \mu(x),\Sigma(x))d \bar{f}  \geq \int_{a}^{b} \mathcal{N}( \bar {f} \vert \underline{\mu} , \underline{\Sigma} ) d \bar{f} \nonumber \\
    &= \frac{1}{2} \left ( \erf \left(\frac{\underline{\mu}-a}{\sqrt{2\underline{\Sigma}}}\right) - \erf \left(\frac{\underline{\mu}-b}{\sqrt{2\underline{\Sigma}}}\right)  \right) \label{eq:inf_gauss_int}
\end{align}
where: $\overline{\mu} = \argmin_{\mu \in [ \mu^L_T, \mu^U_T ]} \vert \mu^m - \mu \vert$ and $\overline{\Sigma}$ is equal to $\Sigma^L_T$ if $\overline{\mu} \in [a,b]$, otherwise  $\overline{\Sigma} = \argmin_{\Sigma \in [\Sigma^L_T, \Sigma^U_T]} \vert \Sigma^m(\overline{\mu}) - \Sigma \vert$. Analogously, for the minimum we have: $\underline{\mu} = \argmax_{\mu \in [ \mu^L_T, \mu^U_T ]} \vert \mu^m - \mu \vert $ and $\underline{\Sigma} =\argmin_{\Sigma \in \{\Sigma^L_T, \Sigma^U_T\}} [\erf(b \vert \underline{\mu}, \Sigma) - \erf(a \vert \underline{\mu}, \Sigma)] $.
\end{proposition}
%\MK{Define erf?}
That is, given lower and upper bounds for the a-posteriori mean and variance in $T$, Proposition \ref{Prop:Gaussian} allows us to analytically bound the $N$ optimisations of Gaussian integrals posed by Equations \eqref{Eq:LowerBOundGenericBidimensional} and \eqref{Eq:UpperBOundGenericBidimensional}. Through this, we can compute values for $\piInfL{T}$ and $\piSupU{T}$,
%\MK{Eqn vs Equation}
which satisfy the LHS of Eqn \eqref{eq:lower_upper_condition1} and the RHS of Eqn \eqref{eq:lower_upper_condition2}. 
Furthermore, note that by definition of $\piInf{T}$ and $\piSup{T}$,
we have that, for every $\bar{x} \in T$, setting $\piInfU{T} = \piSupL{T} = \pi(\bar{x})$ provides values which satisfy the RHS of Eqn \eqref{eq:lower_upper_condition1} and the LHS of Eqn \eqref{eq:lower_upper_condition2} (in the Supplementary Material we discuss how to pick values for $\bar{x}$ to speed up convergence).
Details on the computation of bounds for the a-posteriori mean and variance are discussed in the Supplementary Material.
Interestingly, when the (scaled) probit function is chosen for the likelihood, $\sigma(f)$, then the inference integral over $q(f(x^*)=\bar{f} | \mathcal{D})$ can be expressed in closed form \citep{rasmussen2004gaussian}, which
leads to a simplification of Proposition \ref{Theorem:BOundsGenericBIclass}. Details are given in the Supplementary Material.
\textbf{Branch and Bound Algorithm }
%\subsection{A Branch and Bound Algorithm for \rev{Convergence}}
%\label{Sec:AlgorithmBrench}
%In the previous subsection, we discussed how the extrema of a GPC in a  compact set $T$ can be bounded.
In this paragraph we implement the bounding procedure into a branch and bound algorithm and prove  convergence up to any a-priori specified  $\epsilon > 0 $. 
%Consider a GPC and a compact set $T$. 
We summarise our method for computation of $\piInf{T}$ in Algorithm \ref{alg:bnb_sketch}, which we now briefly describe (analogous arguments hold for $\piSup{T}$).
After initialising $\piInfL{T}$ and $\piInfU{T}$ to trivial values and initialising 
the exploration regions stack $\mathbf{R}$ to the singleton $\{T\}$, the main optimisation loop is entered until convergence (lines 2--9).
Among the regions in the stack, we select the region $R$ with the most promising lower bound (line 3), and refine its lower bounds using Propositions \ref{Theorem:BOundsGenericBIclass} and \ref{Prop:Gaussian} (lines 4--5) as well as its upper bounds through evaluation of points in $R$ (line 6). If further exploration of $R$ is necessary for convergence (line 7), then the region $R$ is partitioned into two smaller regions $R_1$ and $R_2$, which are added to the  regions stack and inherit $R$'s bound values (line 8).
Finally, the freshly computed bounds local to $R \subseteq T$ are used to update the global bounds for $T$ (line 9). Namely, $\piInfL{T}$ is updated to the smallest value among the $\piInfL{R}$ values for $R \in \mathbf{R}$, while $\piInfU{T}$ is set to the lowest observed value yet explicitly computed in line 6.
%
%\vspace{-0.5em}
\begin{spacing}{1.0}
\begin{algorithm} 
\caption{Branch and bound for  $\piInf{T}$}\label{alg:bnb_sketch}
\textbf{Input:} Input space subset $T$; error tolerance $\epsilon>0$; latent mean/variance functions $\mu(\cdot)$ and $\Sigma(\cdot)$ of $q(f(x)=\bar{f} | \mathcal{D})$ \\
\textbf{Output:} Lower and upper bounds on $\piInf{T}$ with $\piInfU{T} - \piInfL{T} \leq \epsilon$
\begin{algorithmic}[1]
\State \emph{Initialisation:} Stack of regions $\mathbf{R} \gets \{ T \}$; \quad $\piInfL{T} \gets - \infty $; \quad $\piInfU{T} \gets + \infty $
\While{ $\piInfU{T} - \piInfL{T} > \epsilon$}
\State Select region $R \in \mathbf{R}$ with lowest bound $\piInfL{R}$
\Statex \hspace{0.5cm}  and delete it from stack
\State Find bounds $[\mu^L_R,\mu^U_R]$ and $[\Sigma^L_R,\Sigma^U_R]$ for latent
\Statex \hspace{0.5cm} mean and variance functions over $R$
%\State $[\Sigma^L_R\Sigma^U_R] \gets $ Solve Eqn \eqref{eq:mean_variance_bounds} in $R$ for ranges of $\Sigma(x)$
%\State $\piInfL(R) \gets 0$
\State Compute $\piInfL{R}$ from $[\mu^L_R,\mu^U_R]$ and $[\Sigma^L_R\Sigma^U_R]$ \Statex \hspace{0.5cm} using Propositions \ref{Theorem:BOundsGenericBIclass} and \ref{Prop:Gaussian}
\State Find $\piInfU{R}$ by evaluating GPC in a point in $R$
\If{$\piInfU{R} - \piInfL{R} > \epsilon$}
\State Split $R$ into two sub-regions $R_1,R_2$, add them 
\Statex \hspace{1cm} to stack and use $\piInfL{R},\piInfU{R}$ as initial 
\Statex \hspace{1cm} bounds for both sub-regions
%\State $\piInfL(R_j)$
\EndIf
%\State $\piInfU(T) \gets \min \left( \piInfU(T), \piInfU(R) \right)$;  $\piInfL(T) \gets \max$ of $\piInfL(R^*)$ for  $R^*$ in $\mathbf{R}$ 
%\If{$\piInfU{R} < \piInfU{T}$}
\State Update $\piInfL{T}$ and $\piInfU{T}$ with current best 
\Statex \hspace{0.5cm} bounds found 
%\State Update current best bounds $\piInfL{T},\piInfU{T}$ as 
%\Statex \hspace{0.5cm} lowest (highest) $\piInfL{R}$ ($\piInfU{R}$)
%\EndIf
\EndWhile
%\State $n^C  $
\State \textbf{return} $[\piInfL{T},\piInfU{T}]$
%\WHILE{$ n < n_{\max}$}
%\STATE $w  sample from $p(w | )$
%\\ENDWHILE
%\STATE return $\hat{p}$
\end{algorithmic}
\end{algorithm}
\end{spacing}
%\vspace{-0.5em}
%In order to prove that the algorithm terminates in a finite time, we show that for any $\epsilon$
For our approach to work, it is crucial that Algorithm 1 converges, i.e.\ that the loop of lines $2-9$ terminates.
Given an a-priori specified threshold $\epsilon$, Theorem \ref{TH:convergenceBrenchAndBOund} ensures that there exists a latent space discretisation such that the bounding error (i.e.\ the difference between the upper and lower bound) vanishes.
Thanks to the properties of branch and bound algorithms \citep{balakrishnan1991branch}, this guarantees that our method %for the computation of adversarial global \MK{we don't discuss global} robustness of GPC will 
converges in finitely many iterations.
%The following theorem guarantees the convergence of our method to the actual values up to any a-priori specified error threshold $\epsilon$.
%\begin{theorem}
%\label{TH:convergenceBrenchAndBOund}
%Assume $\mu:\mathbb{R}^d \to \mathbb{R}$ and $\Sigma:\mathbb{R}^d \to \mathbb{R}$ are Lipschitz continuous in $T\subseteq \mathbb{R}^{m}$.
%Then, for any $\epsilon >0$ and $x \in T$, there exists a partition of the latent space $S$, and a radius of the branch and bound $r$ \LL{TO be defined} such that $$  |\pi(x) - \pi^{inf}(S,r)|\leq \epsilon \quad  \wedge \quad |\pi(x) - \pi^{sup}(S,r)|  \leq \epsilon,$$
%\end{theorem}
%\LL{Proof to be updated.}
\begin{theorem}
\label{TH:convergenceBrenchAndBOund}
Assume $\mu:\mathbb{R}^d \to \mathbb{R}$ and $\Sigma:\mathbb{R}^d \to \mathbb{R}$ are Lipschitz continuous in $T\subseteq \mathbb{R}^{m}$.
Then, for $\epsilon >0$, there exists a partition of the latent space $\mathcal{S}$ and $r>0$  such that, for every $R \subseteq T$ of side length of less than $r$, it holds that $\vert \piInfU{R} - \piInfL{R} \vert \leq \epsilon $ and  $\vert \piSupU{R} - \piSupL{R} \vert  \leq \epsilon.$
\end{theorem}

\paragraph{Computational Complexity}

Proposition \ref{Prop:Gaussian} implies that the bounds in Proposition \ref{Theorem:BOundsGenericBIclass} can be obtained in $\mathcal{O}(N)$, with $N$ being the number of intervals the real line is being partitioned into (this scales like $\frac{1}{\epsilon}$, as discussed in the proof of Theorem \ref{TH:convergenceBrenchAndBOund}).
Computation of $\mu^L_T$ and $\mu^U_T$ is performed in  $\mathcal{O}(\vert \mathcal{D} \vert),$ while obtaining $\Sigma^U_T$ involves the solution of a convex quadratic problem in $d + \vert \mathcal{D} \vert$ variables, where $d$ is the dimension of the input space.
Solving for $\Sigma^L_T$ requires the solution of $2 \vert \mathcal{D} \vert + 1$ linear programming problems in $d + \vert \mathcal{D} \vert$ dimensions.
Refining through branch and bound has a worst-case cost exponential in the number of non-trivial dimensions of $T$.
The CPU time required for convergence of our method is analysed in the Supplementary Material.

\section{MULTICLASS CLASSIFICATION}
\label{Sec:Multiclass}
In this section we show how the results for binary classification can be generalised to the multi-class case.
Given a class index $c \in \{1,\ldots, C \},$ we are interested in computing upper and lower bounds on $\pi^c(x | \mathcal{D})$ for every $x \in T$.
In order to do so, we extend Proposition \ref{Theorem:BOundsGenericBIclass} to the multi-class case in Proposition \ref{Theorem:BOundsGenericMulticlass}, and show that the resulting multi-dimensional integrals can be reduced to the two-class case by marginalisation (Proposition \ref{Prop:MultiCLass}). 
\begin{proposition}
\label{Theorem:BOundsGenericMulticlass}
%Let $\mathcal{S}=\{S_i\subseteq \mathbb{R}^{C}, i \in \{1,...N \}\}$ be a fine partition of $\mathbb{R}^{C}$. Then, for $c \in \{1,...,N \}$, it holds that
Let $\mathcal{S}=\{S_i \mid i \in \{1,...N \}\}$ be a finite partition of $\mathbb{R}^{C}$ (the latent space). Then, for $c \in \{1,...,C \}$:
%\begin{align}   \label{Eq:LowerBOundGenericMulticlass}
%     \sum_{i=1}^{N}  \inf_{x\in S_i}\sigma^c(x) \inf_{x \in T} \int_{S_i} \mathcal{N}(\bar f | \mu(x),\Sigma(x))d \bar{f} \leq \pi^c(x^* | \mathcal{D}) \leq  \sum_{i=1}^{N}  \sup_{x\in S_i}\sigma^c(x) \sup_{x \in T} \int_{S_i} \mathcal{N}(\bar f  | \mu(x),\Sigma(x))d \bar{f}
%\end{align}
%
\begin{align}
  %\label{Eq:LowerBOundGenericMulticlass}
    &\pi_{\min}^c(T) \geq  \sum_{i=1}^{N}  \min_{x\in S_i}\sigma^c(x) \min_{x \in T} \int_{S_i} \mathcal{N}(\bar f | \mu(x),\Sigma(x))d \bar{f} \nonumber
   \\
    %\label{Eq:UpperBOundGenericMulticlass}
    &\pi_{\max}^c(T) \leq  \sum_{i=1}^{N}  \max_{x\in S_i}\sigma^c(x) \max_{x \in T} \int_{S_i} \mathcal{N}(\bar f  | \mu(x),\Sigma(x))d \bar{f} \nonumber
\end{align}
\end{proposition}
Proposition \ref{Theorem:BOundsGenericMulticlass} guarantees that, for all $x\in T$, $\pi^c(x | \mathcal{D})$  can be upper- and lower-bounded by solving $2 N$ optimisation problems.
%As shown in the Supplementary Material,  under the assumption that $\sigma$ is the softmax function and that $S_i$ is an axis-parallel hyper-rectangle, % assumption that can always be enforced,
%$\sup_{x \in S_i} \sigma^c(x)$ and $\inf_{x \in S_i} \sigma^c(x)$ can be computed by simply evaluating the vertices of $S_i$.
In Proposition \ref{Prop:MultiCLass}, we show that upper and lower bounds for the integral of a multi-dimensional Gaussian distribution, such as those appearing in Proposition \ref{Theorem:BOundsGenericMulticlass}, can be obtained by optimising uni-dimensional integrals over both the input and latent space.
%Each of these integrals can be optimised by employing Proposition \ref{Prop:Gaussian}.
In what follows, we call $\mu_{i:j}(x)$ the subvector of $\mu(x)$ containing only the components from $i$ to $j$, and similarly we define $\Sigma_{i:k,j:l}(x)$, the submatrix of $\Sigma(x)$ containing rows from $i$ to $k$ and columns from $j$ to $l$. 
%Then, the following proposition follows.
\begin{proposition}
\label{Prop:MultiCLass}
Let $S=\prod_{i=1}^{C}[ k_{i}^1, k_{i}^2]$ be an axis-parallel hyper-rectangle.
For $i \in \{1,\ldots,C-1 \}$ and $f \in \mathbb{R}^{C-1-i}$, define $\mathcal{I}:=i+1:C$ and: %$\mathcal{I}:=i:i$ and $I:=i+1:|C|$.  Call 
\begin{align*}
    \mu^f_i(x)=\mu_i(x)-\Sigma_{i,\mathcal{I}}(x)\Sigma_{\mathcal{I},\mathcal{I}}^{-1}(f-\mu_{\mathcal{I}}(x))\\
    \Sigma^f_i(x)= \Sigma_{i,i}(x)-\Sigma_{i,\mathcal{I}}(x)\Sigma_{\mathcal{I},\mathcal{I}}^{-1}\Sigma_{i,\mathcal{I}}^T(x).
\end{align*}
Let $S^{i+1}=\prod_{j=i+1}^{C} [ k_{j}^{1},k_{j}^{2}]$, then we have that:
\begin{align*}
\max_{x\in T}\int_{S}&\mathcal{N}(z|\mu(x),\Sigma(x)) \leq  \max_{x\in T} \int_{k_{C}^1}^{k_{C}^2} \hspace*{-0.3 cm}\mathcal{N}(z|\mu_{C}(x),\\ & \Sigma_{C,C}(x) )dz  \prod_{i=1}^{C-1}  \max_{\substack{x\in T \\ f\in S^{i+1}}} \hspace*{-0.2 cm} \int_{k_{i}^1}^{k_{i}^2} \hspace*{-0.2 cm} \mathcal{N}(z|\mu^f_i(x),\Sigma^f_i(x) )dz \\%\end{align*}
%\begin{align*} 
\min_{x\in T}\int_{S}&\mathcal{N}(z|\mu(x),\Sigma(x)) \geq \min_{x\in T} \int_{k_{C}^1}^{k_{C}^2} \hspace*{-0.3 cm} \mathcal{N}(z|\mu_{C}(x), \\
& \Sigma_{C,C}(x) )dz  \prod_{i=1}^{C-1}  \min_{\substack{x\in T \\ f\in S^{i+1}}} \hspace*{-0.2 cm} \int_{k_{i}^1}^{k_{i}^2} \hspace*{-0.2 cm} \mathcal{N}(z|\mu^f_i(x),\Sigma^f_i(x) )dz.
%&\inf_{x\in T}\int_{S}\mathcal{N}(z|\mu(x),\Sigma(x)) \geq\\
%&\quad\quad  \inf_{x\in T} \int_{k_{C,1}}^{k_{C,2}}\mathcal{N}(z|\mu_{C}(x),\Sigma_{C,C|}(x) )dz  \prod_{i=1}^{C-1}\inf_{x\in T,f\in S^{i+1}} \int_{k_{i,1}}^{k_{i,2}}\mathcal{N}(z|\mu^f(x,i),\Sigma^f(x,i) )dz 
\end{align*}
%where $S^{i+1}=\prod_{j=i+1}^{C} [ k_{j,1},k_{j,2}]$.
\end{proposition}
Proposition \ref{Prop:MultiCLass} reduces the computation of the bounds for the multi-class case to a product of extrema of univariate Gaussian distributions for which Proposition \ref{Prop:Gaussian} can be iteratively applied.
Analogously to what we discussed for the binary case, the resulting bound can be refined through a branch and bound algorithm to ensure convergence up to any desired tolerance $\epsilon > 0 $.
Notice that the computational complexity for the multi-class case is exponential in $C$.
\section{EXPERIMENTAL RESULTS}
%\ap{Work in progress}
\label{Sec:ExperimentalResults}

\begin{figure}[h]
	\centering
	{\hspace*{-.4cm}  \includegraphics[width = 0.24\textwidth]{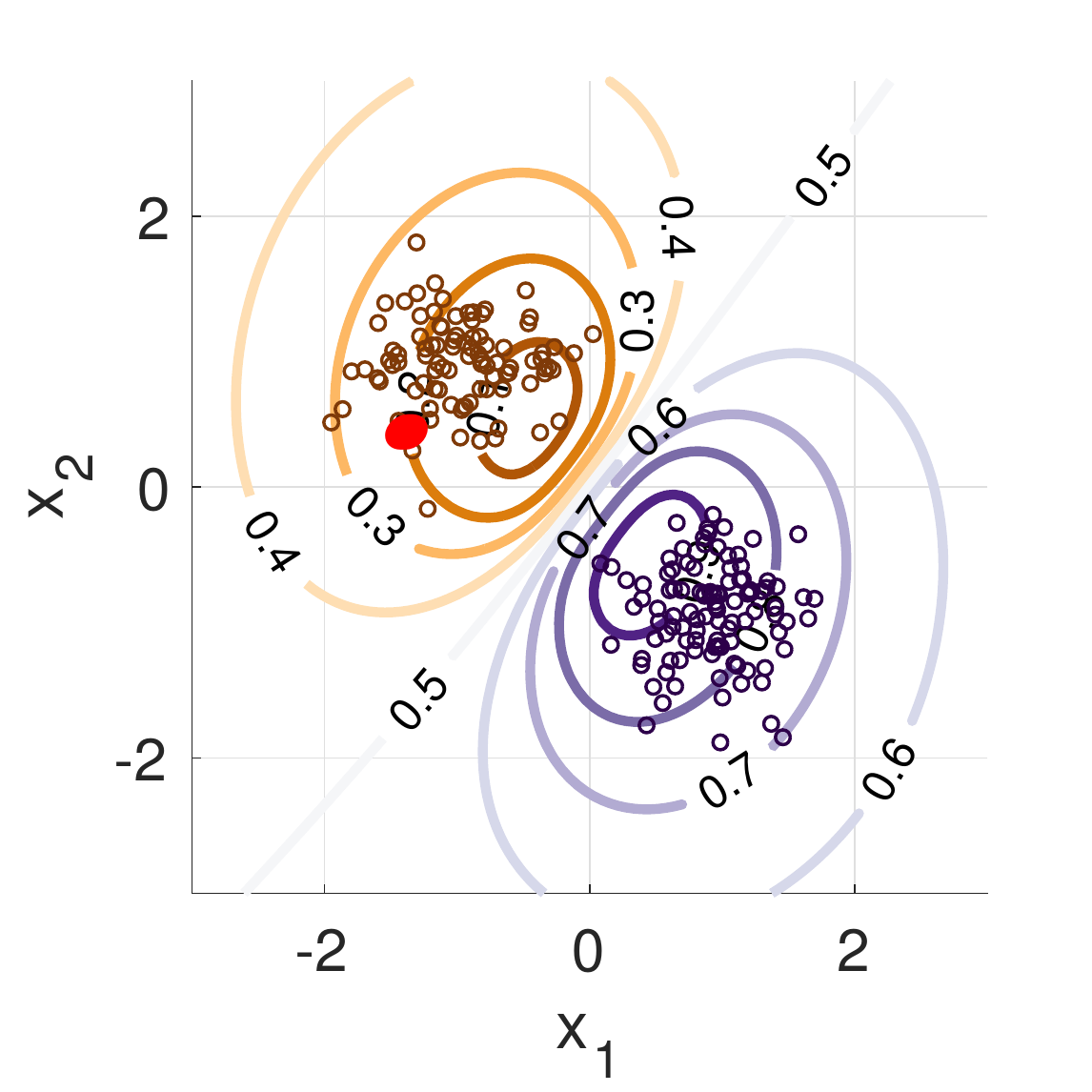}}
	{\hspace*{-.2cm}  \includegraphics[width = 0.24\textwidth]{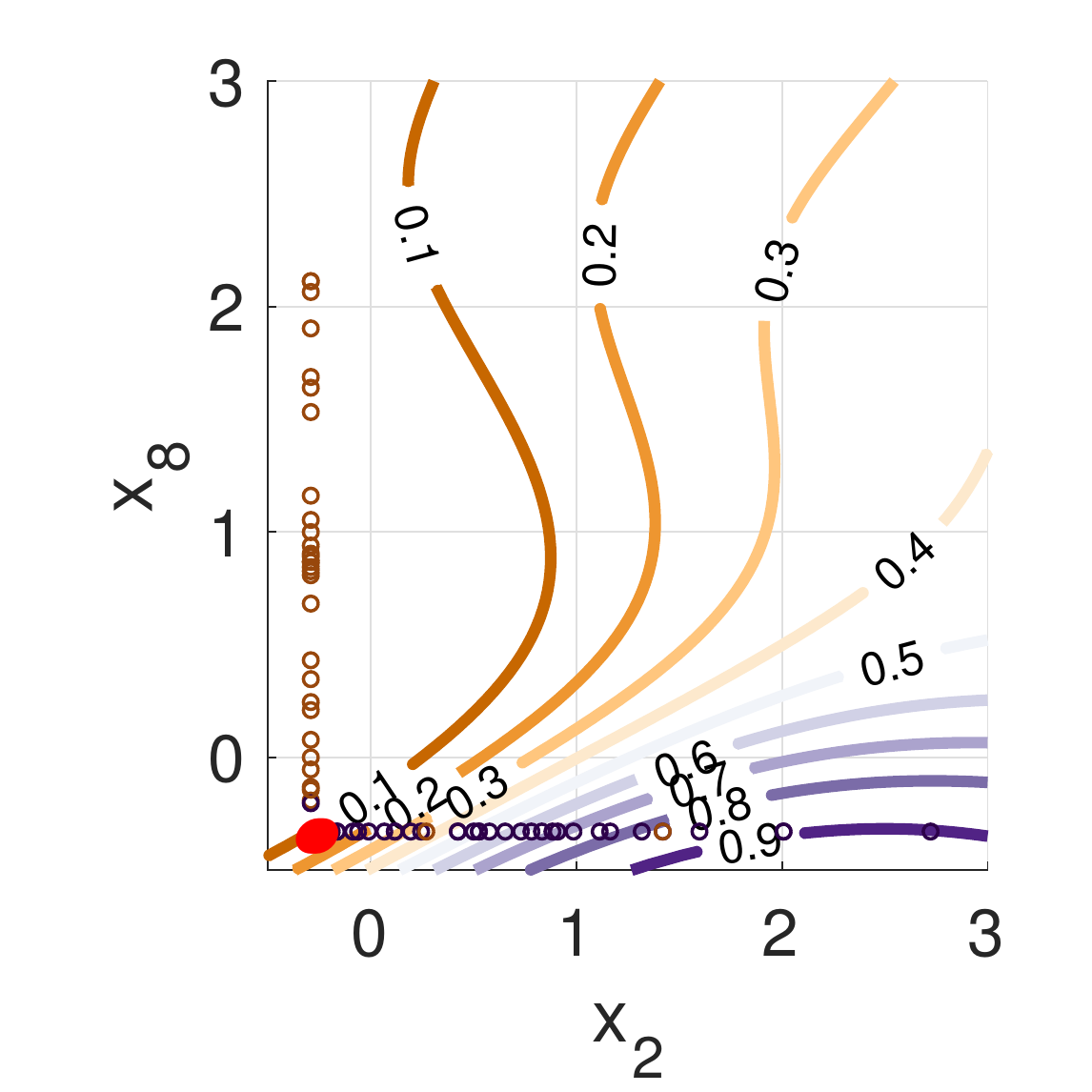}}\\
	{
	\hspace*{0cm} \includegraphics[width = 0.23\textwidth]{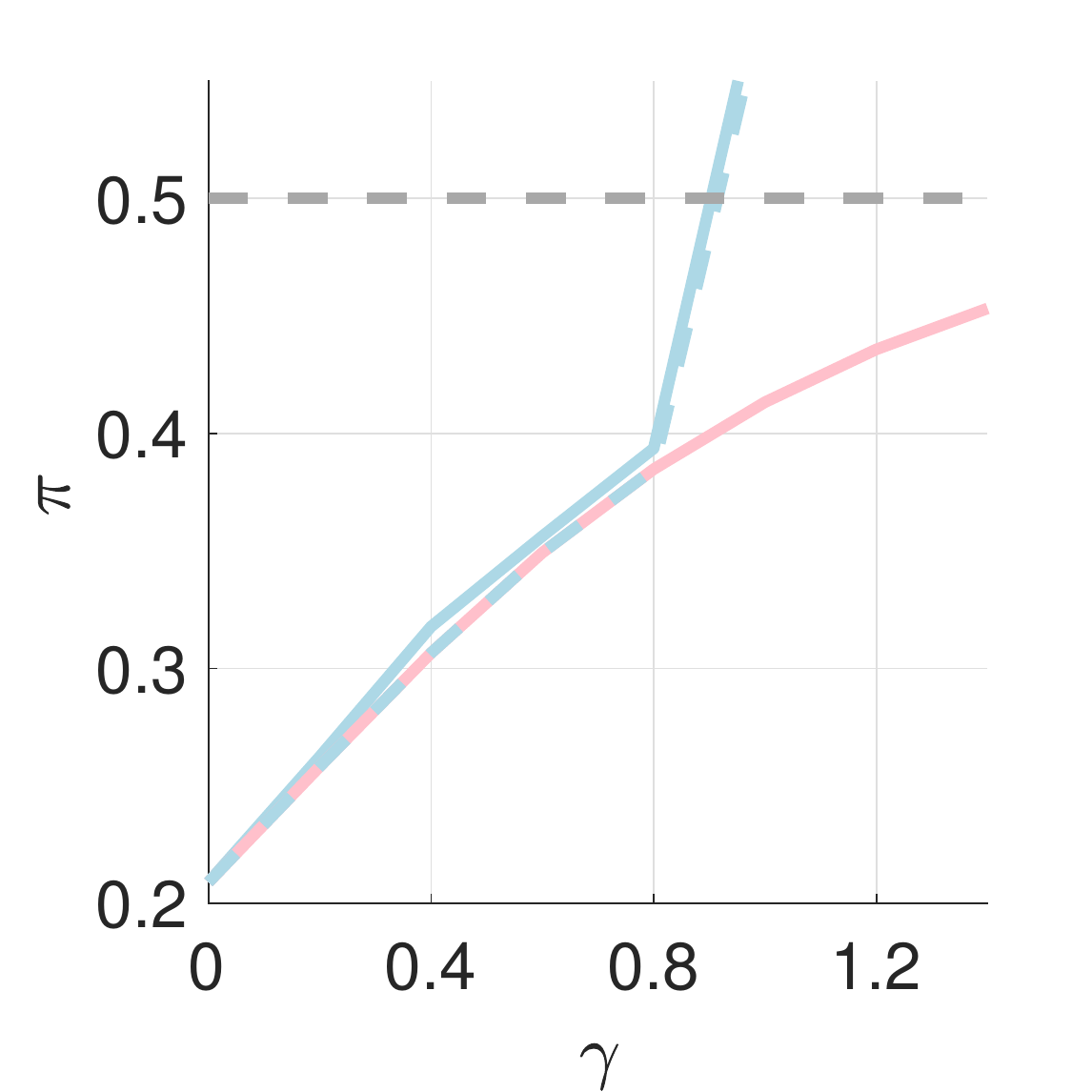}}
	\hspace*{0cm} {\includegraphics[width = 0.23\textwidth]{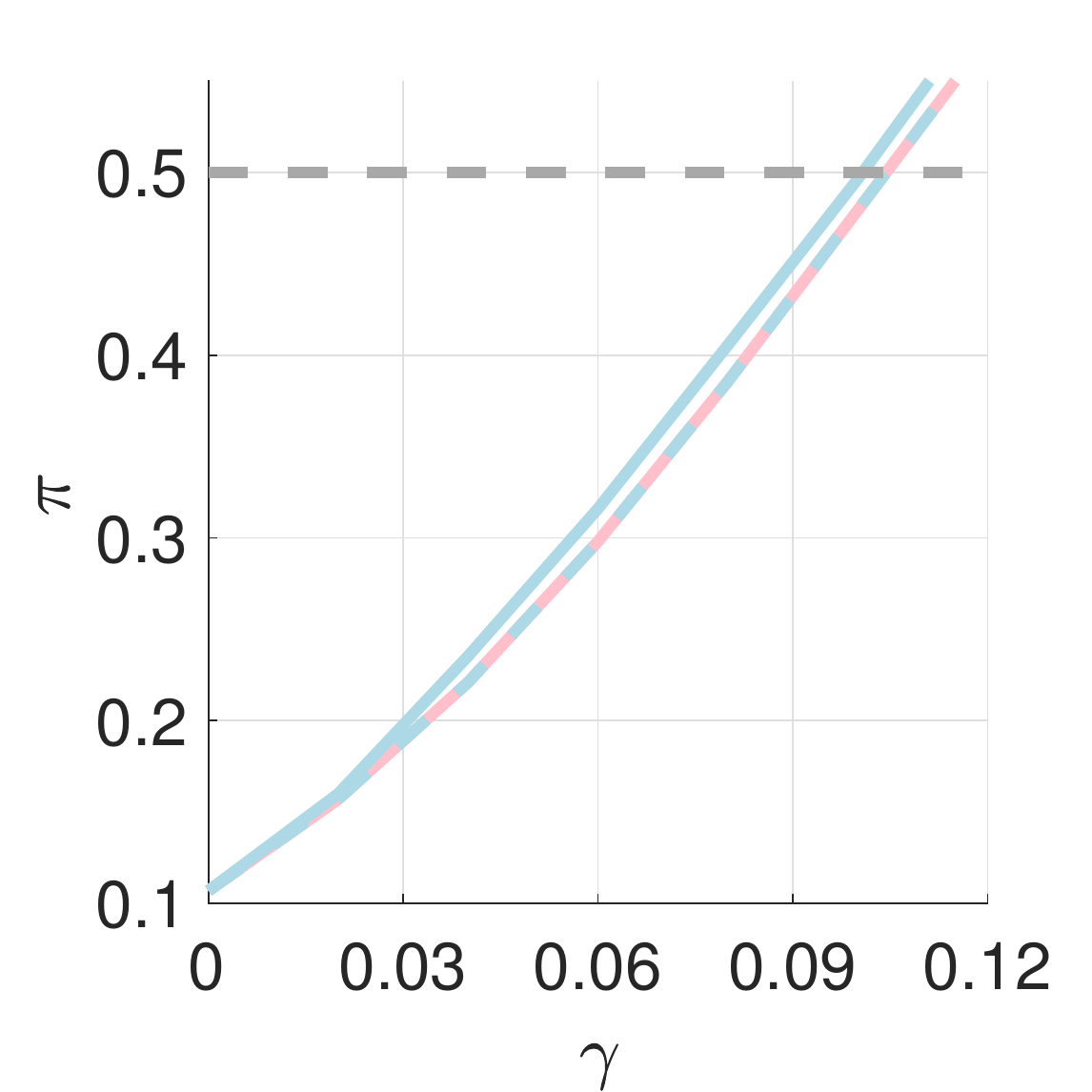}}
	{\hspace*{0cm}  \includegraphics[width = 0.45\textwidth]{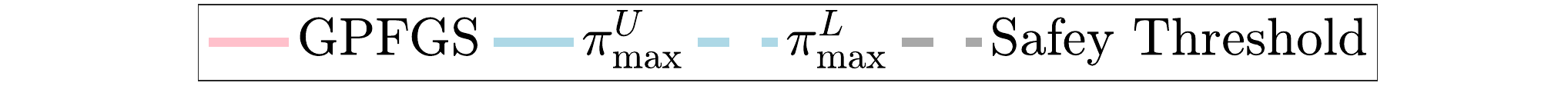}} 
	\caption{\textbf{First row}: Contour plot and test points for Synthetic2D (left); projected contour plot and test points for 2 dimensions of SPAM (right, dimensions $2$ and $8$ as selected by $L_1$-penalised logistic regression); red dots mark selected test points. \textbf{Second row}: Safety analysis for the two selected test point. Shown are the upper and lower bounds for tolerance $\epsilon = 0.02$ on $\piSup{T}$ (solid and dashed blue curves) and the GPFGS adversarial attack (pink curve). %The true values of $\piSup{T}$ are guaranteed to be between the solid and dashed blue curves.
	}
	\label{fig:safety1}
\end{figure}
\begin{figure}[h]
	\centering
	{\hspace*{0.2cm}  \includegraphics[width = 90pt]{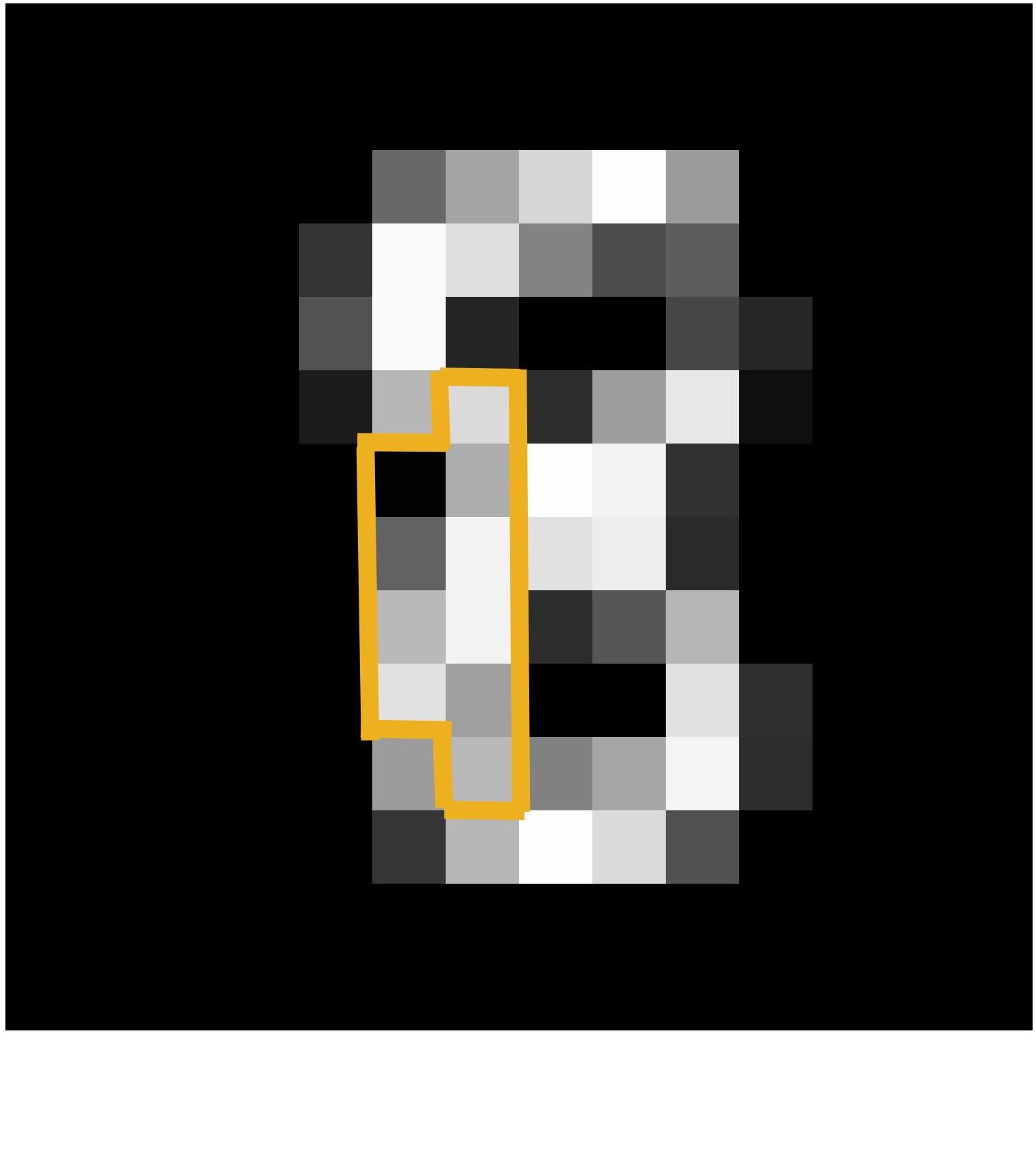}}
	{  \includegraphics[width = 90pt]{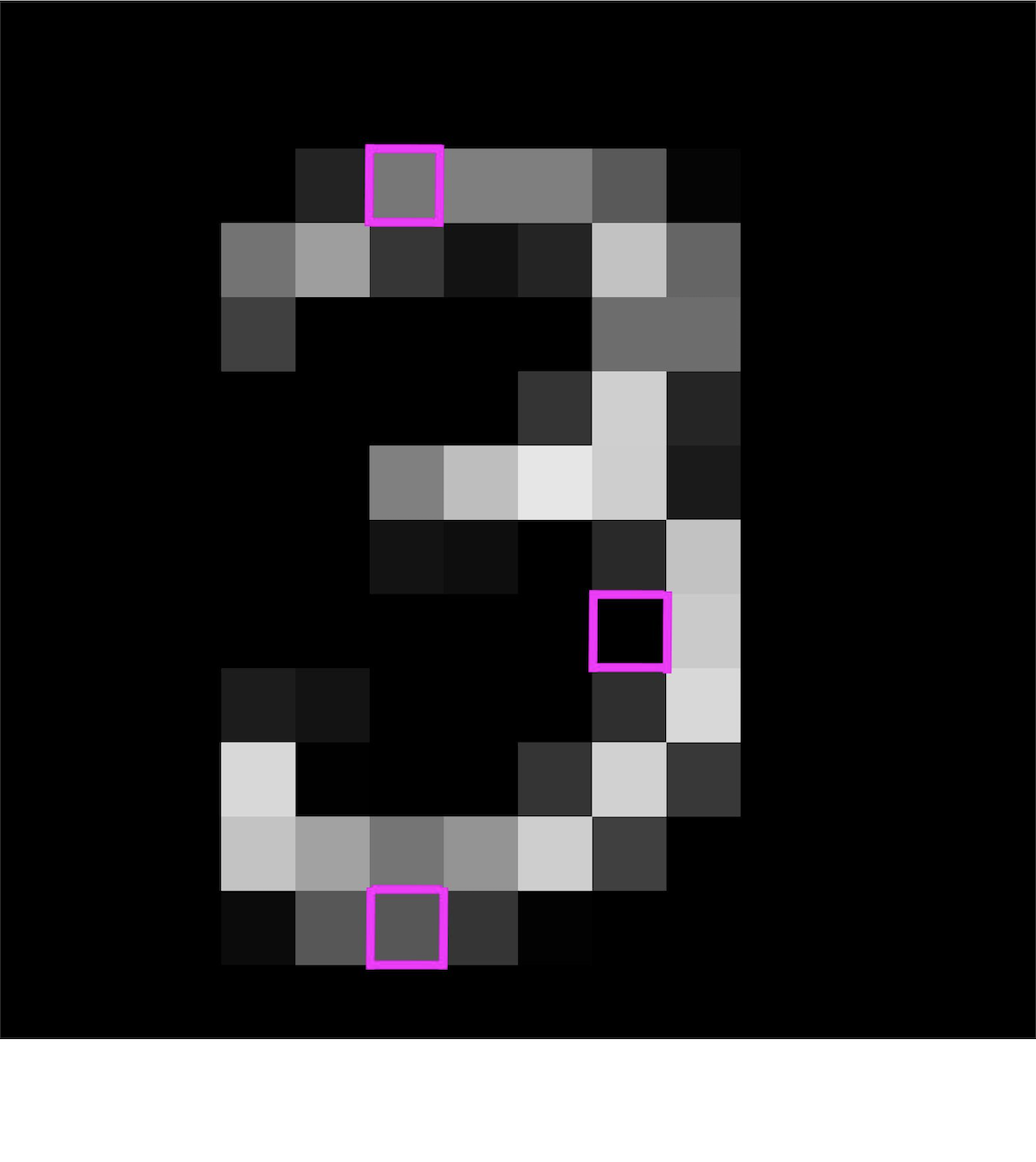}}\\
	\hspace*{0cm} {\includegraphics[width = 0.23\textwidth]{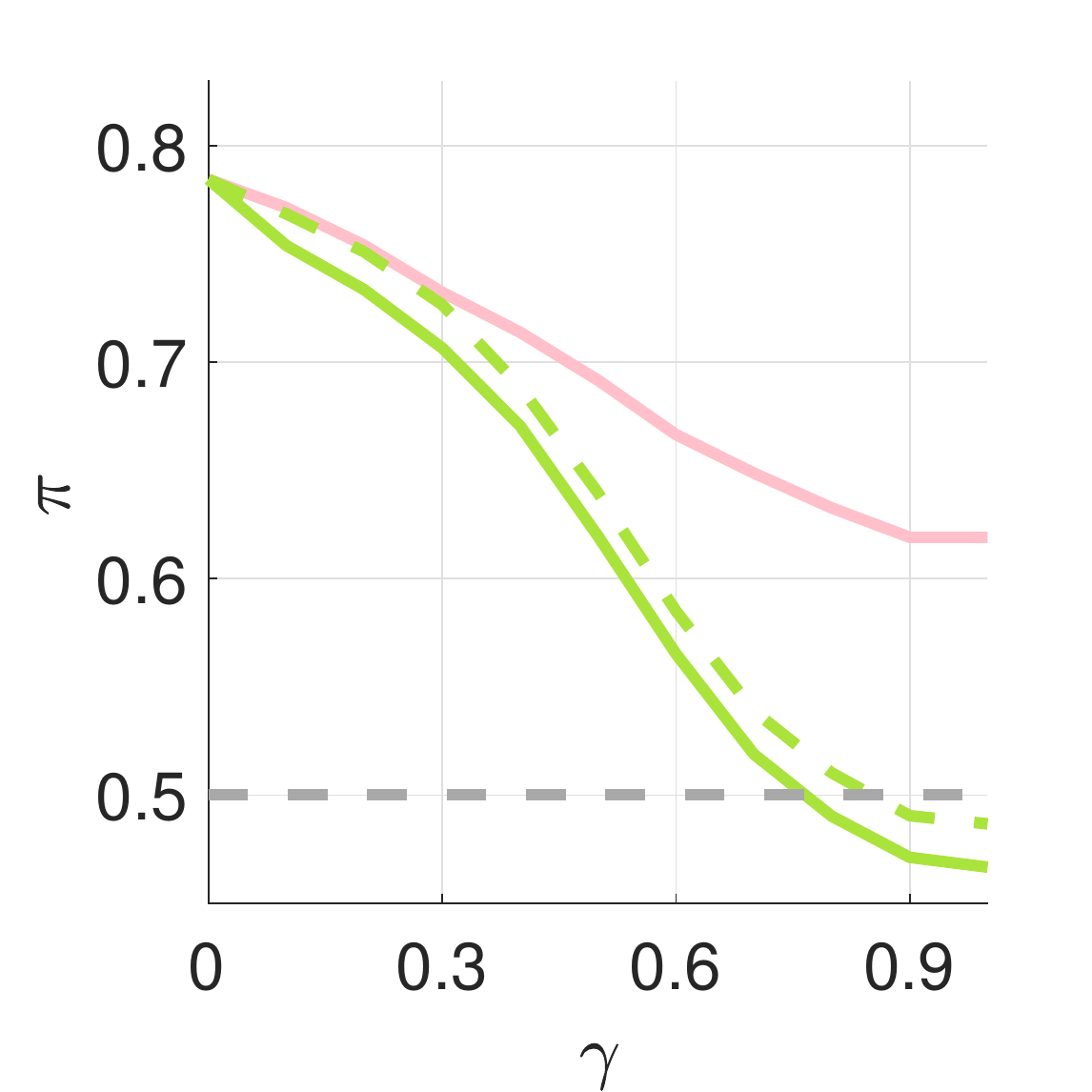}}
	\hspace*{0cm} {\includegraphics[width = 0.23\textwidth]{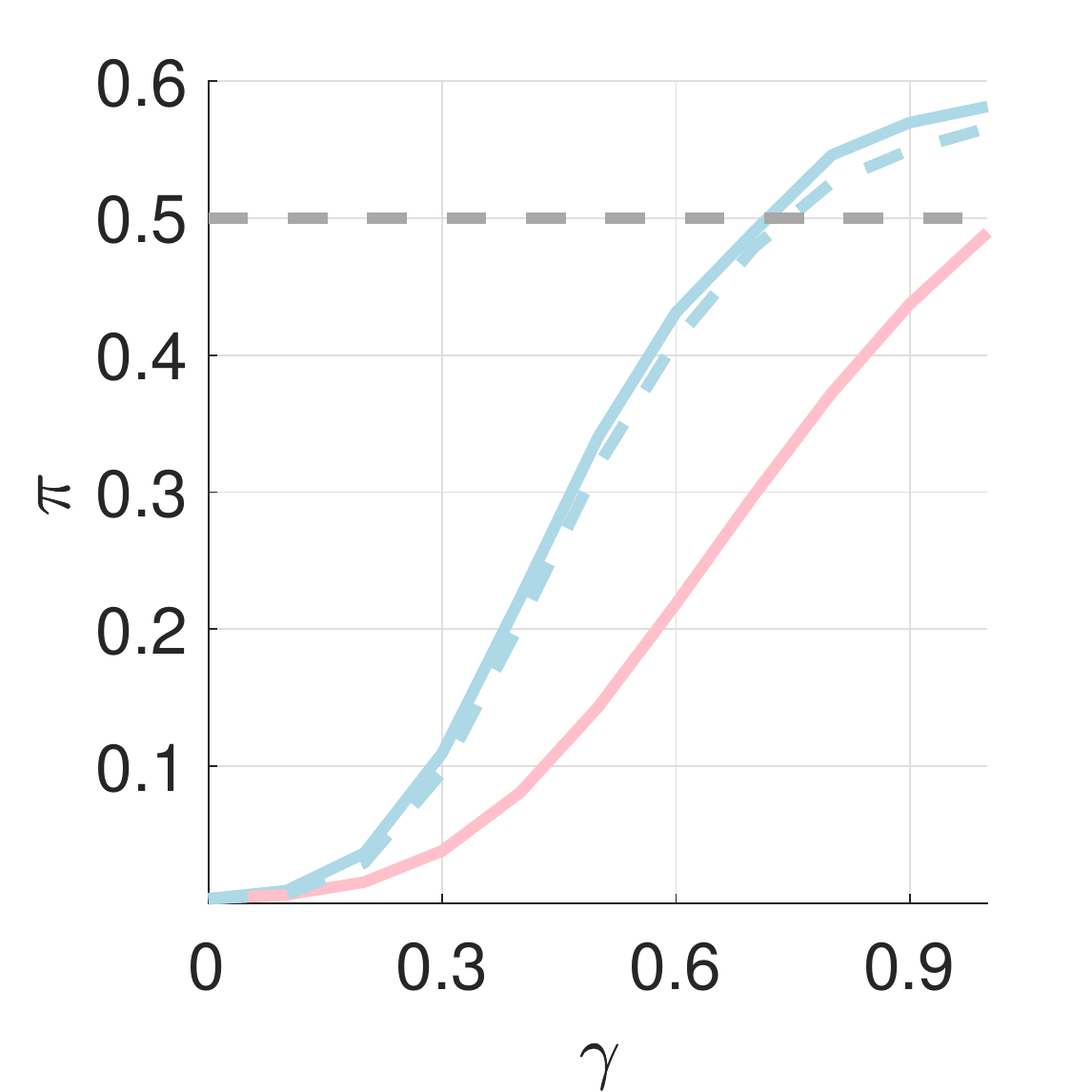}}
	{\hspace*{0cm}  \includegraphics[width = 0.45\textwidth]{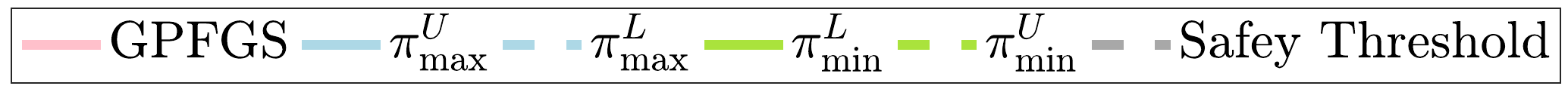}} 
	\caption{\textbf{First row}: Sample of 8 from MNIST38 along with 10 pixels selected by SIFT (left) and sample of 3 from MNIST38 along with the 3 pixels that have the shortest lengthscales after GPC training (right). \textbf{Second row}: Safety analysis for the two images. Shown are the upper and lower bounds for $\epsilon = 0.02$ on either $\piSup{T}$ or $\piInf{T}$ (solid and dashed blue resp.\ green curves) and the GPFGS adversarial attack (pink curve). %The true values of $\piSup{T}$ or $\piInf{T}$ are guaranteed to be between the solid and dashed blue (green) curves.
	}
	\label{fig:safety2}
\end{figure}

We employ our methods to experimentally analyse the robustness profile of GPC models in adversarial settings\footnote{Code: \textit{https://github.com/andreapatane/check-GPclass}}.
%In Section \ref{subsec:safety} we check for adversarial local safety in neighborhoods around specific test points, and compare the guarantees we provide with an adversarial attack for GPC models \citep{grosse2017wrong}.
%In Section \ref{subsec:robustness} we empirically compare different training regimes for Bayesian inference with GPC, analysing the effect of the approximate Bayesian techniques used and hyper-parameter fitting.
%Finally, in Section \ref{subsec:interpretability} we show how adversarial local robustness measures can be used for interpretability analysis, and compare our results with those obtained by LIME \citep{ribeiro2016should}.
We give results for three datasets: (i) Synthetic2D, generated by shifting a two-dimensional standard-normal either along the first dimension (class 1) or the second one (class 2); (ii) the SPAM dataset \citep{Dua:2019}; (iii) a subset of the MNIST dataset \citep{lecun1998mnist} with classes $3$ and $8$ (MNIST38) and a subset with classes $3$, $5$ and $8$ (MNIST358).
For scalability, results for MNIST38 are given for feature-level analysis (as done in \citet{ruan2018reachability} for deep networks).
Namely, we analyse either salient patches detected by SIFT \citep{lowe2004distinctive} or we select the relevant pixels corresponding to the shortest GP length-scales. %Our approach is independent of the particular feature selection method chosen.
%We evaluate our methods on three datasets.
%The first, Synthetic 2D, is generated by shifting two-dimensional standard-normals either along the first dimension (Class 1) or along the second dimension (Class 2). Secondly, we use the SPAM dataset from the UCI database \citep{Dua:2019}. Lastly, we use the MNIST38 subset (which contains only classes $3$ and $8$) and the MNIST358 subset (only classes $3$, $5$ and $8$) of the original MNIST dataset \citep{lecun1998mnist}.
%Test set accuracies obtained are $100 \%$, $93 \%$ and $94-98 \%$ respectively.
%Training was performed using the probit likelihood for the binary classification, and the softmax in the multiclass case. \rev{As we are interested in worst-case guarantuees, we use the $L_{\infty}$-norm induced metric when defining $\gamma$-balls around test points, as by definition these contain the $\gamma$-balls of any other $L$-norm induced metric.}
%
%
\subsection{Adversarial Local Safety}
\label{subsec:safety}
%For ARNO: I just re-arrenged a bit the discussion, and water down a bit the final claim of the subsection.
We depict the local adversarial safety results for four points selected from the Synthethic2D, SPAM, and MNIST38 datasets in Figures \ref{fig:safety1} and \ref{fig:safety2}. %- first rows of the plots show the test points selected, robustness results are given in the second row. 
%That is we check for the existence of adversarial examples in neighborhood regions of $\gamma > 0$ for increasing values of $\gamma$ (x-axis in the second row).
%For each value of $\gamma$ we compute the class probabilities extrema $\piInf{T}$ and $\piSup{T}$ and check 
To this end, we set $T \subseteq \mathbb{R}^d$ to be a $L_\infty$ $\gamma-$ball around the chosen test point %\footnote{That is, $T=\{ x \in \mathbb{R}^d \,s.t.\, |x-x^*|_{\infty} \leq \gamma \}$ for test point $x^*$. Note that any other $L_P$ norm could have been equally used.}
and iteratively increase $\gamma$ (x-axis in the second row plots), checking whether there are adversarial examples in $T$.
Namely, if the point is originally assigned to class 1 (respectively class 2) we check whether the minimum classification probability in $T$ is below the decision boundary threshold, that is, if $\piInf{T} < 0.5$ (resp. $\piSup{T} > 0.5$). %\footnote{Note that, unlike analyses purely based on the latent mean, our approach can be used for decision boundaries different from $0.5$ (e.g.\ one vs. all classification or robust decision making).}.
We compare the values provided by our method (blue solid and dashed line for class 2, green solid and dashed line for class 1) with GPFGS \citep{grosse2018limitations}, a gradient based heuristic attack for GPC (pink line).
%
%\begin{figure}[h]
%	\centering
%	{\hspace*{0.2cm}  \includegraphics[width = 90pt]{figures/FIg8_new_copy.jpg}}
%	{  \includegraphics[width = 90pt]{figures/3figNIPSpurple.png}}\\
%	\hspace*{0cm} {\includegraphics[width = 0.23\textwidth]{figures/mnist8_new.eps}}
%	\hspace*{0cm} {\includegraphics[width = 0.23\textwidth]{figures/mnist_safety_trail2.eps}}
%	{\hspace*{0cm}  \includegraphics[width = 0.45\textwidth]{figures/legend_fig2_cropped.pdf}} 
%	\caption{\textbf{First row}: Sample of 8 from MNIST38 along with 10 pixels selected by SIFT (left) and sample of 3 from MNIST38 along with the 3 pixels that have the shortest lengthscales after GPC training (right). \textbf{Second row}: Safety analysis for the two images. Shown are the upper and lower bounds for $\epsilon = 0.02$ on either $\piSup{T}$ or $\piInf{T}$ (solid and dashed blue resp.\ green curves) and the GPFGS adversarial attack (pink curve). %The true values of $\piSup{T}$ or $\piInf{T}$ are guaranteed to be between the solid and dashed blue (green) curves.
%	}
%	\label{fig:safety2}
%\end{figure}
%
%
Naturally, as $\gamma$ increases, the neighborhood region $T$ becomes larger, hence the confidence for the initial class can decrease.
Interestingly, while our method succeeds in finding adversarial examples in all cases shown (i.e.\ both the lower and upper bound on the computed quantity cross the decision boundary), the heuristic attack fails to find adversarial examples in the Synthetic2D and in the MNIST38 case.
This happens as GPFGS builds on linear approximations of the GPC function, hence failing to find solutions to Eqn \eqref{Eq:infSubROsustness} when there are non-linearities.
In particular, near the point selected for the Synthetic2D dataset (red dot in the contour plot) the gradient of the GPC points away from the decision boundary.
Hence, no matter the value of $\gamma$, GPFGS will not go above 0.5 in this case (pink line of the bottom-left plot). 
On the other hand, for the SPAM dataset, the GPC model is locally linear around the selected test point (red dot in top right contour plot).
%As such, following the gradient around it will lead to values close to optimality.
Interestingly, the MNIST38 examples (Figure \ref{fig:safety2}) provide results analogous to those of Synthetic2D.
While our method finds adversarial examples on both occasions, GPFGS fails to do so (even with $\gamma = 1.0$ which is the maximum region possible for normalised pixel values).

\subsection{Adversarial Local Robustness}
\label{subsec:robustness}
We evaluate the empirical distribution of $\delta$-robustness (see Definition \ref{ProbDef:Robustness}) on $50$ randomly selected test points for each of the three datasets considered.
That is, given $T$, we compute $\delta =  \pi_{\max}(T) - \pi_{\min}(T) $.
Notice that a smaller value of $\delta$ implies a more robust model. 
In particular, we analyse how the GPC model robustness is affected by the training procedure used. 
We compare the robustness obtained when using either the Laplace or the EP posterior approximations technique.
Further, we investigate the influence of the number of marginal likelihood evaluations (epochs) performed during hyper-parameter optimisation on robustness.

\begin{figure*}[h]
	\centering
	{\includegraphics[width = 0.28\textwidth]{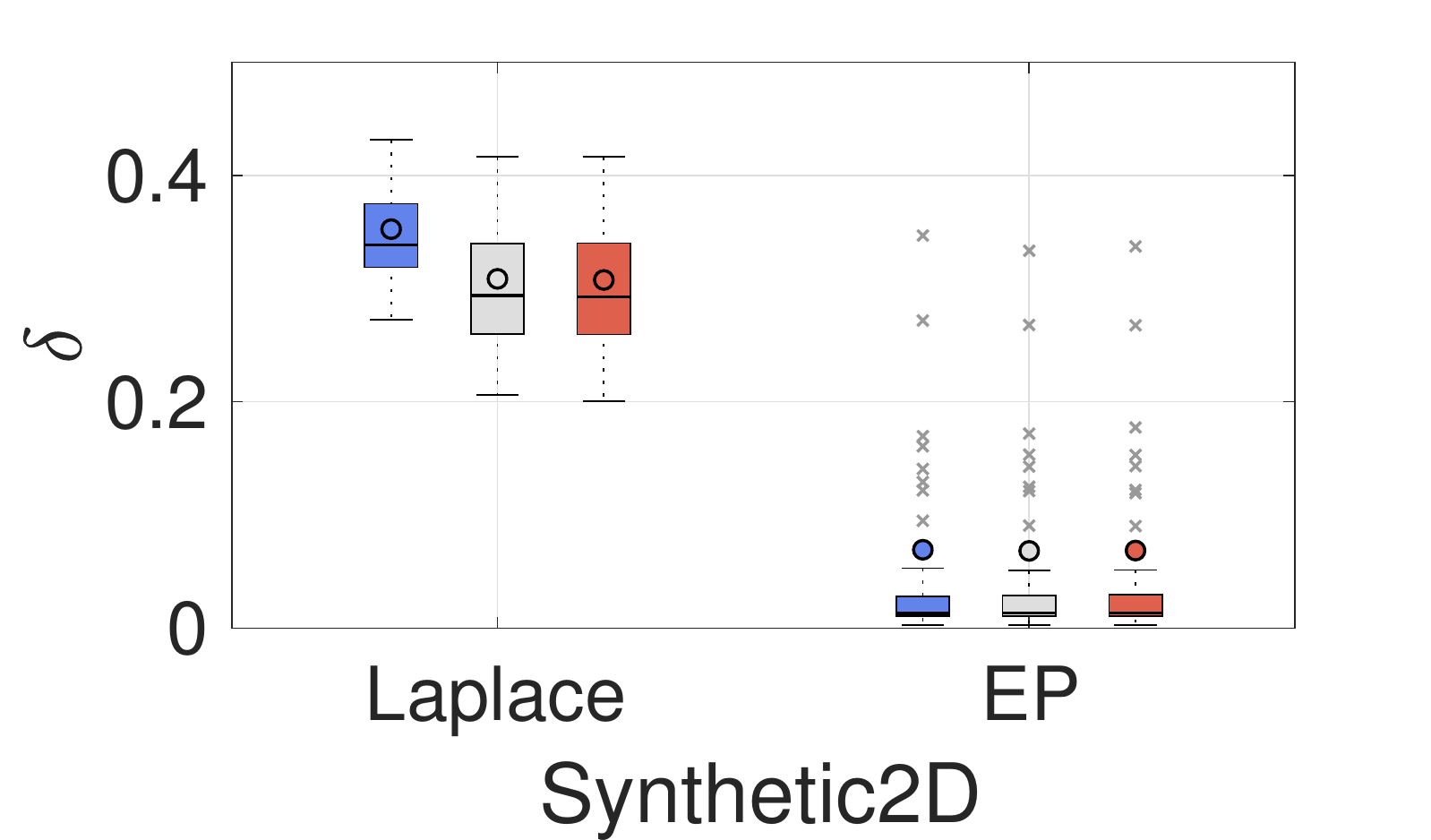}
	\includegraphics[ width = 0.28\textwidth]{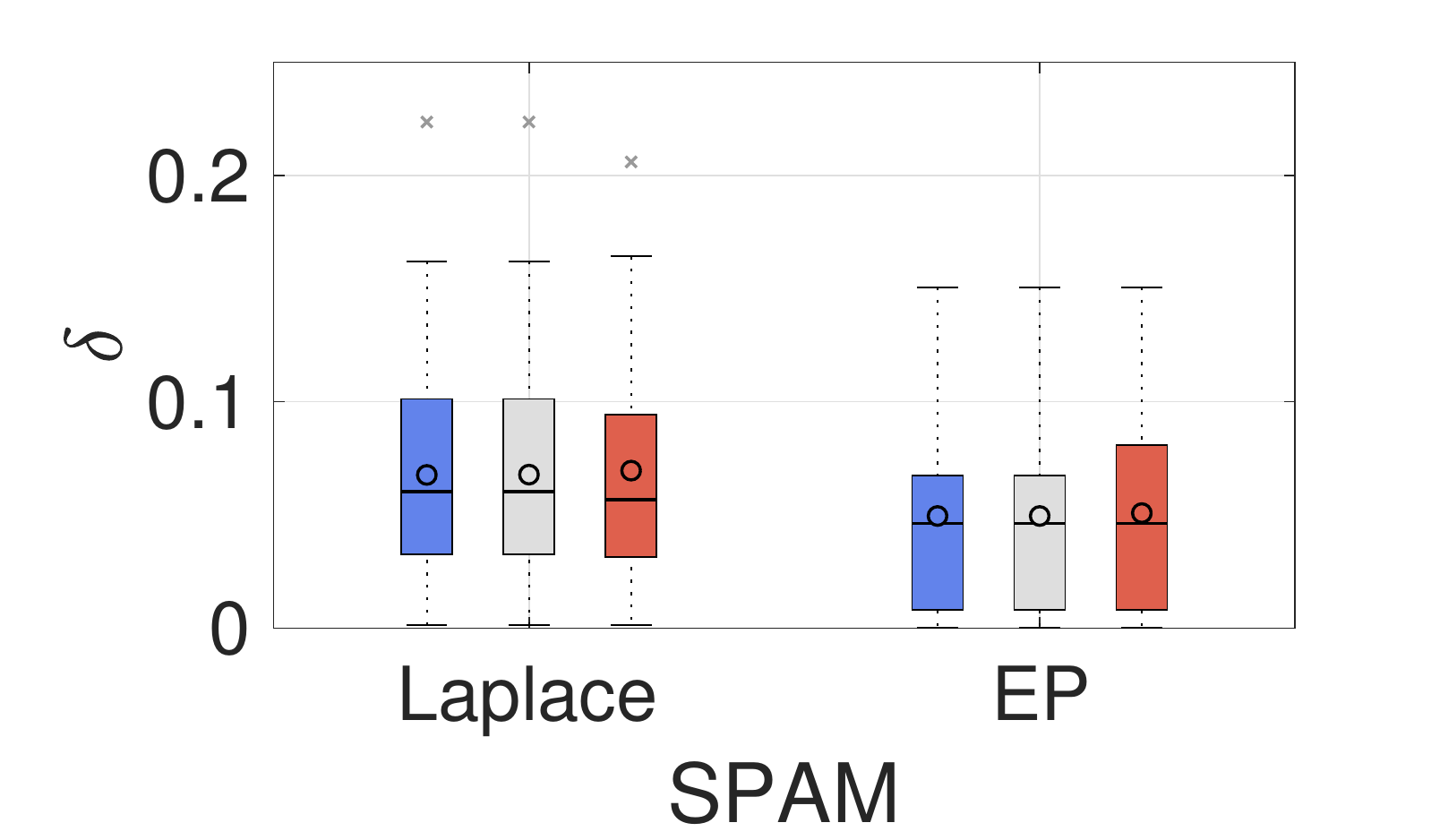}
	\includegraphics[ width = 0.28\textwidth]{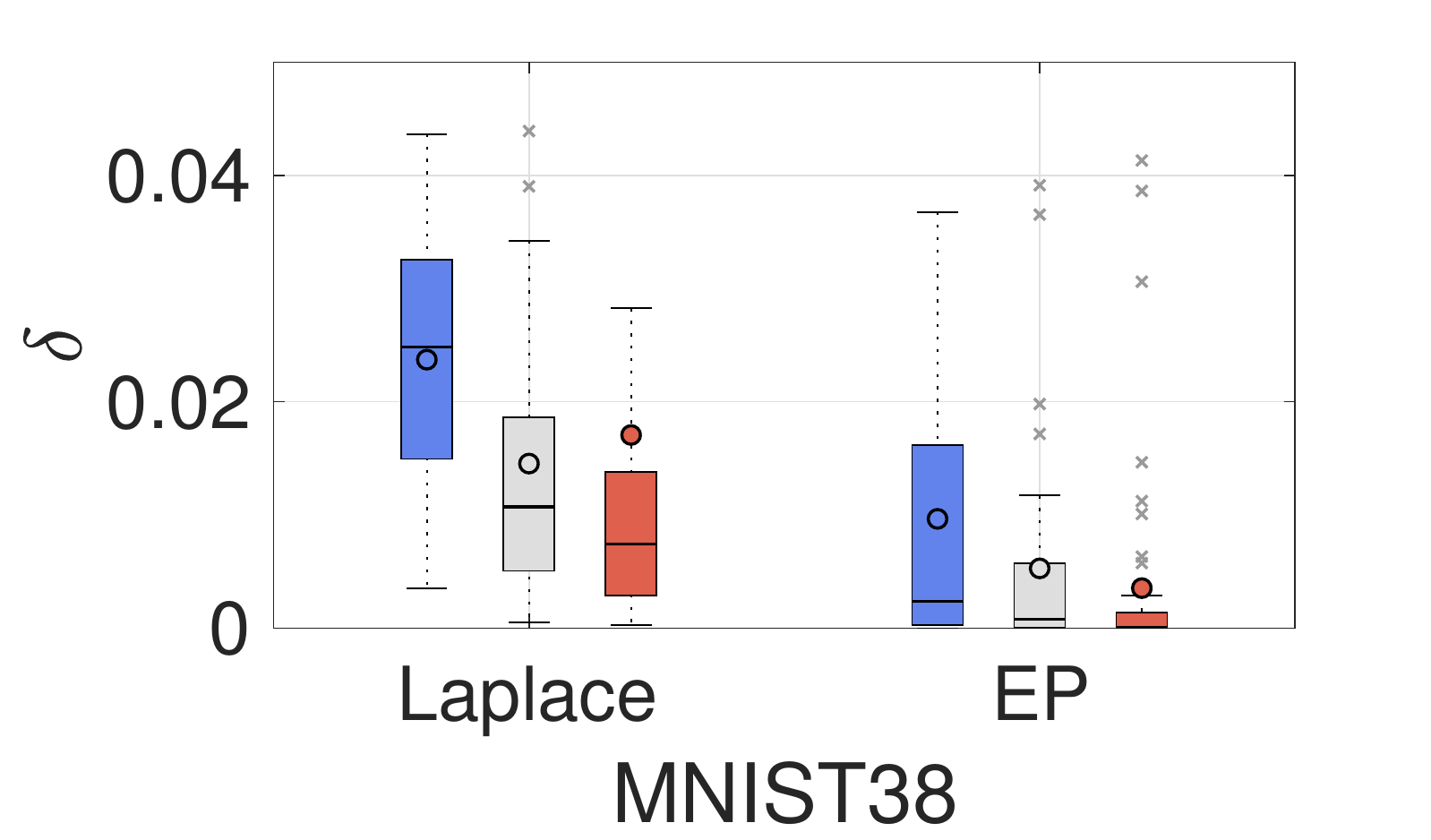}}
	 {\hspace*{0.2cm} \includegraphics[width = 0.32\textwidth]{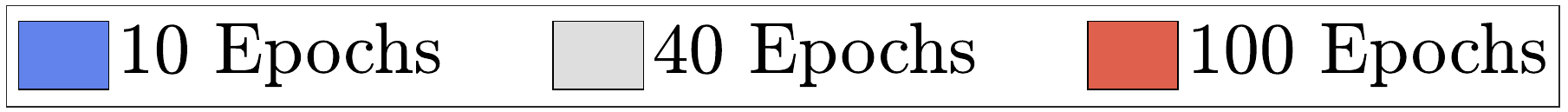}}
	 
	\caption{Boxplots for the distribution of robustness on the three datasets, %considered in this paper,
	comparing Laplace and EP approximation. %(a smaller $\delta$ implies a more robust model). 
	}
	\label{fig:Robustness}
\end{figure*}

Results are depicted in Figure \ref{fig:Robustness}, %, using $\gamma = 0.5$, $\gamma = 0.03$ and $\gamma = 0.1$ respectively for the Synthetic 2D, the SPAM and the MNIST dataset \rev{(so that Algorithm \ref{alg:bnb_sketch} terminates in reasonable time)}, and comparing results for $10$, $40$ and $100$ hyper-parameter optimisation \rev{epochs}. 
for $10$, $40$ and $100$ hyper-parameter optimisation epochs. 
Note that the analyses for the MNIST38 samples are restricted only to the most influential SIFT feature, %one of them, 
and thus $\delta$ values for MNIST38 are %sensibly 
smaller in magnitude than %that 
for the other two datasets (for which all the input variables are simultaneously changed).
Interestingly, %even though these results are to be considered empirical in nature, we systematically observe 
this empirical analysis demonstrates that
GPCs trained with EP are consistently more robust than those trained using Laplace.
In fact, for both Synthetic2D and MNIST38, EP yields a model about 5 times more robust than Laplace.
For SPAM, the difference in robustness is the least pronounced.
While Laplace approximation works by local approximations, EP calibrates mean and variance estimation by a global approach, which generally results in a more accurate approximation \citep{rasmussen2004gaussian}.
We compare Laplace and EP posterior approximations with that made by Hamiltonian Monte Carlo (HMC) - that is, as in \citet{minka2001expectation} we use HMC as gold standard. 
The empirical distances found on the posterior approximation w.r.t.\ HMC are on average as follows (smaller values are better): (i) Synthetic2D - Laplace: 1.04, EP: 0.14; (ii) SPAM - Laplace: 0.35, EP: 0.32; (iii) MNIST38 - Laplace: 0.52, EP: 0.32. 
This shows a correlation between the robustness and the posterior approximation quality in the datasets considered.
These results quantify and confirm for GPCs that a more refined estimation of the posterior is beneficial for model adversarial robustness \citep{cardelli2019statistical}. 
Interestingly, the values of $\delta$ decrease as the number of training epochs increases, thus robustness improves with training epochs.
This is in contrast to what is observed in the deep learning literature \citep{tsipras2018robustness}.
More training in the Bayesian settings may imply better calibration of the latent mean and variance function to the observed data. 
%Finally, the $\delta$-robustness values obtained for the SPAM datasets are quite stable across all the training procedures explored here, with small differences that are not statistically significant. Since the SPAM dataset is almost perfectly linearly separable, we in fact obtain a similar GPC independently of the training procedure. % applied.

\subsection{Interpretability Analysis}
\label{subsec:interpretability}
Finally, we show how adversarial robustness can be used for interpretability analysis for GPC models.
We provide comparison with pixel-wise  LIME \citep{ribeiro2016should}, a model-agnostic interpretability technique that relies on local linear approximations.
Given a test point $x^*$ consider the one-sided intervals $T^i_{\gamma}(x^*) = [x^*, x^* + \gamma e_i]$ (with $e_i$ being the vector of $0$s except for $1$ at dimension $i$).
We compute how much the maximum and minimum values can change over the one-sided intervals in both directions:
\begin{align*}
    \mathbf{\Delta}^i_{\gamma}(x^*) = &\left(\piSup{T^i_{\gamma}(x^*)} -\piSup{T^i_{-\gamma}(x^*)}\right)\\
    & + \left(\piInf{T^i_{\gamma}(x^*)}-\piInf{T^i_{-\gamma}(x^*)}\right). \label{eq:interpretcenter}
\end{align*}
Intuitively, this provides a non-linear generalisation of numerical gradient estimation (more details in Supplementary Material) which is close to the metric used in \citet{ribeiro2016should} as $\gamma$ shrinks to $0$.
While $\mathbf{\Delta}^i_{\gamma}(x^*)$ is local to a given $x^*$, following LIME, global interpretability information is obtained by averaging local results over $M$ test points, i.e.\ by computing  $\mathbf{\Delta}^i_{\gamma} = \frac{1}{M} \sum_{j = 1}^{M} \mathbf{\Delta}^i_{\gamma}(x^j)$.

%Finally, we show how the robustness measure introduced in Definition \ref{ProbDef:Robustness} can be used for pointwise interpretability analysis and compare our findings with those obtained by pixel-wise interpretability using LIME \citep{ribeiro2016should}, a popular model-agnostic interpretability technique using local linear approximations. To this end, we introduce an interpretability metric that can be seen as a non-linear generalisation of numerical gradient estimation. Defining one-sided intervals $T^i_{\gamma}(x^*) = [x^*, x^* + \gamma*e_i]$ (with $e_i$ being the vector of $0$s except for $1$ at dimension $i$), we compare how much the maximum and minimum values can change over the one-sided intervals in both directions (similarly to symmetric differences in numerical gradient estimation). This means we compute
%\begin{align*}
%    \mathbf{\Delta^i_{\gamma}}(x^*) &= &[\piSup{T^i_{\gamma}(x^*)} -\piSup{T^i_{-\gamma}(x^*)}]\\
%    && + [\piInf{T^i_{\gamma}(x^*)}-\piInf{T^i_{-\gamma}(x^*)}]. \label{eq:interpretcenter}
%\end{align*}
%\begin{equation*}
%        \mathbf{\Delta^i_{\gamma}} = \frac{1}{M} \sum_{j = 1}^{M} \mathbf{\Delta^i_{\gamma}}(x^j)
%\end{equation*}
%for a global analysis, following LIME's approach of aggregating local insights into a global insight by averaging over a selection of $M$ test points.

\begin{figure}[h]
  	\centering
	\includegraphics[clip = on, trim = 10mm 10mm 10mm 10mm ,width = 0.125\textwidth]{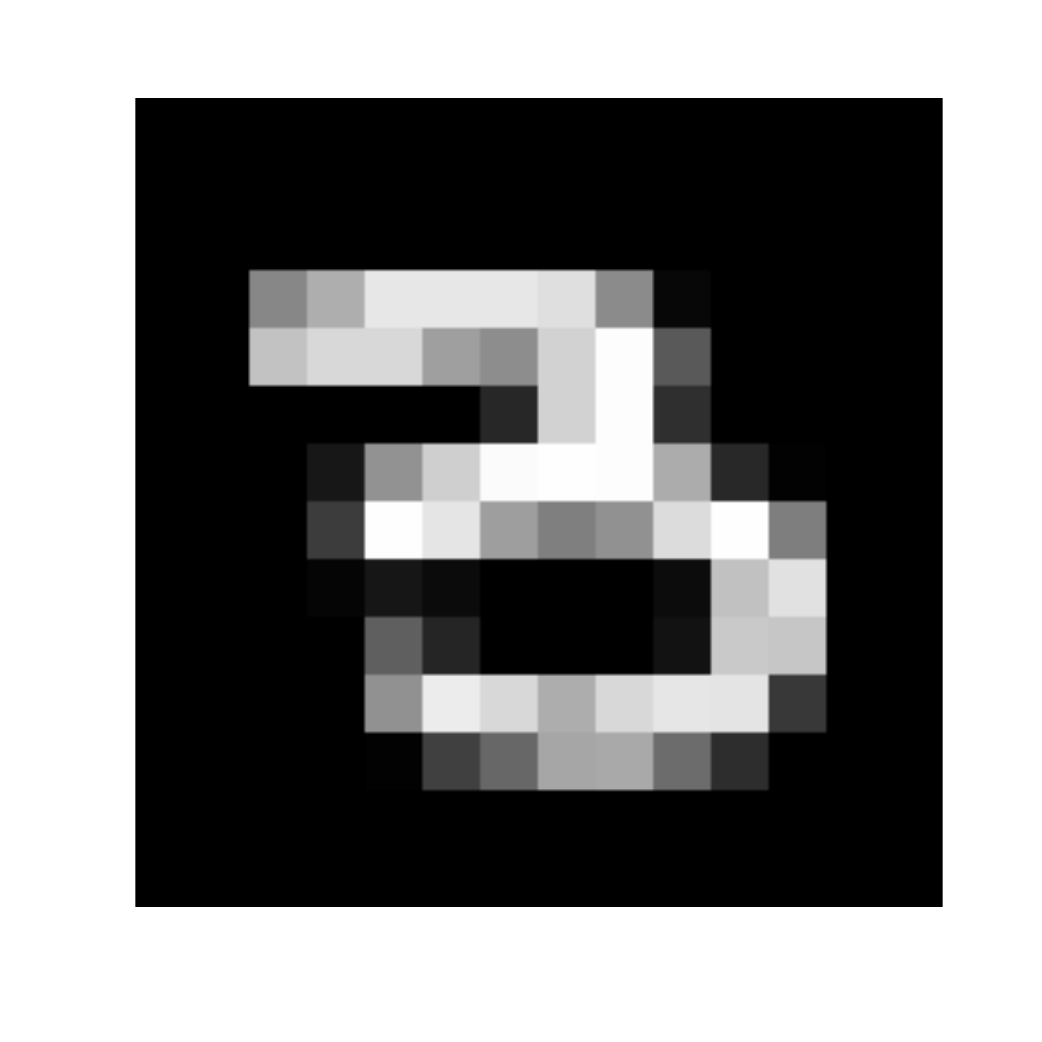}
	\includegraphics[clip = on, trim = 10mm 10mm 10mm 10mm ,width = 0.125\textwidth]{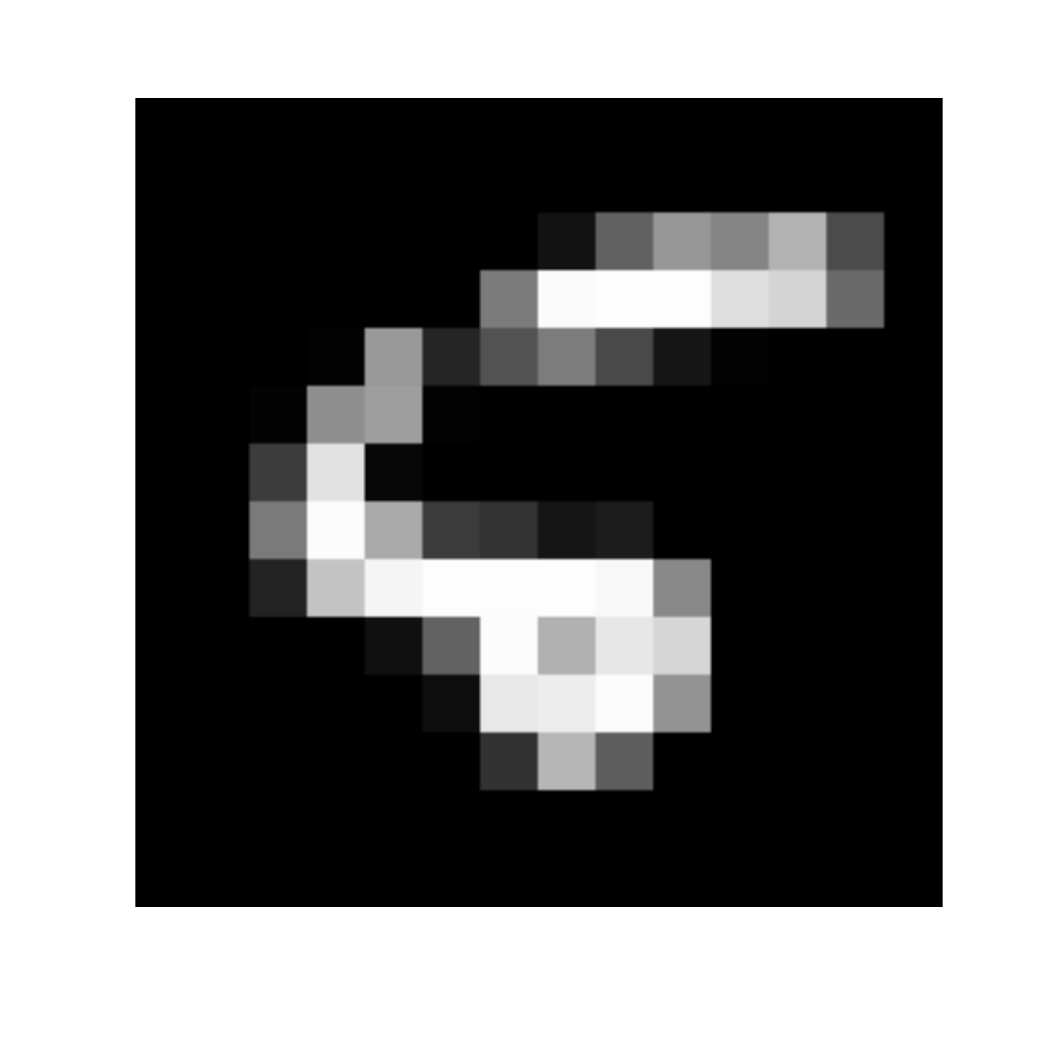}
	\includegraphics[clip = on, trim = 10mm 10mm 10mm 10mm ,width = 0.125\textwidth]{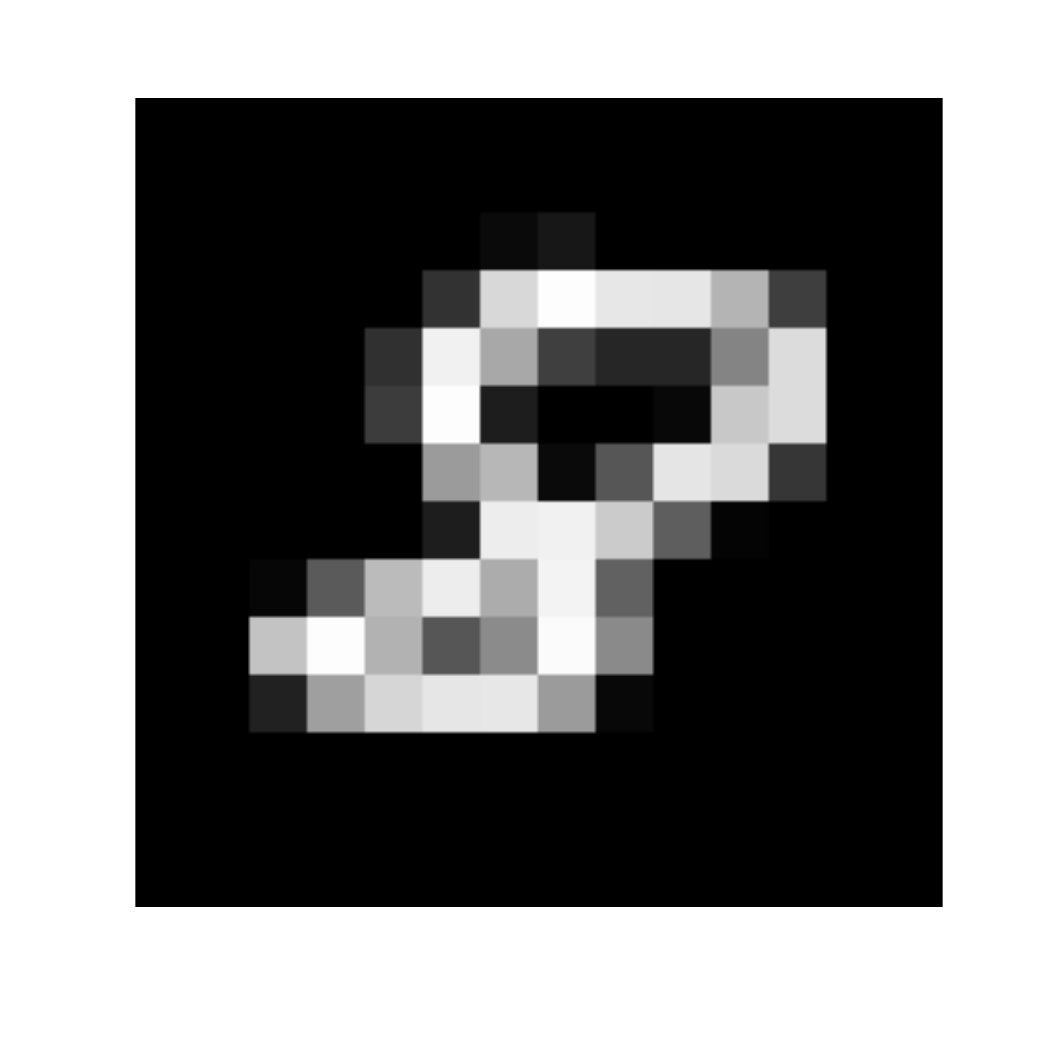} \\ 
	
	\includegraphics[clip = on, trim = 10mm 10mm 10mm 10mm ,width = 0.125\textwidth]{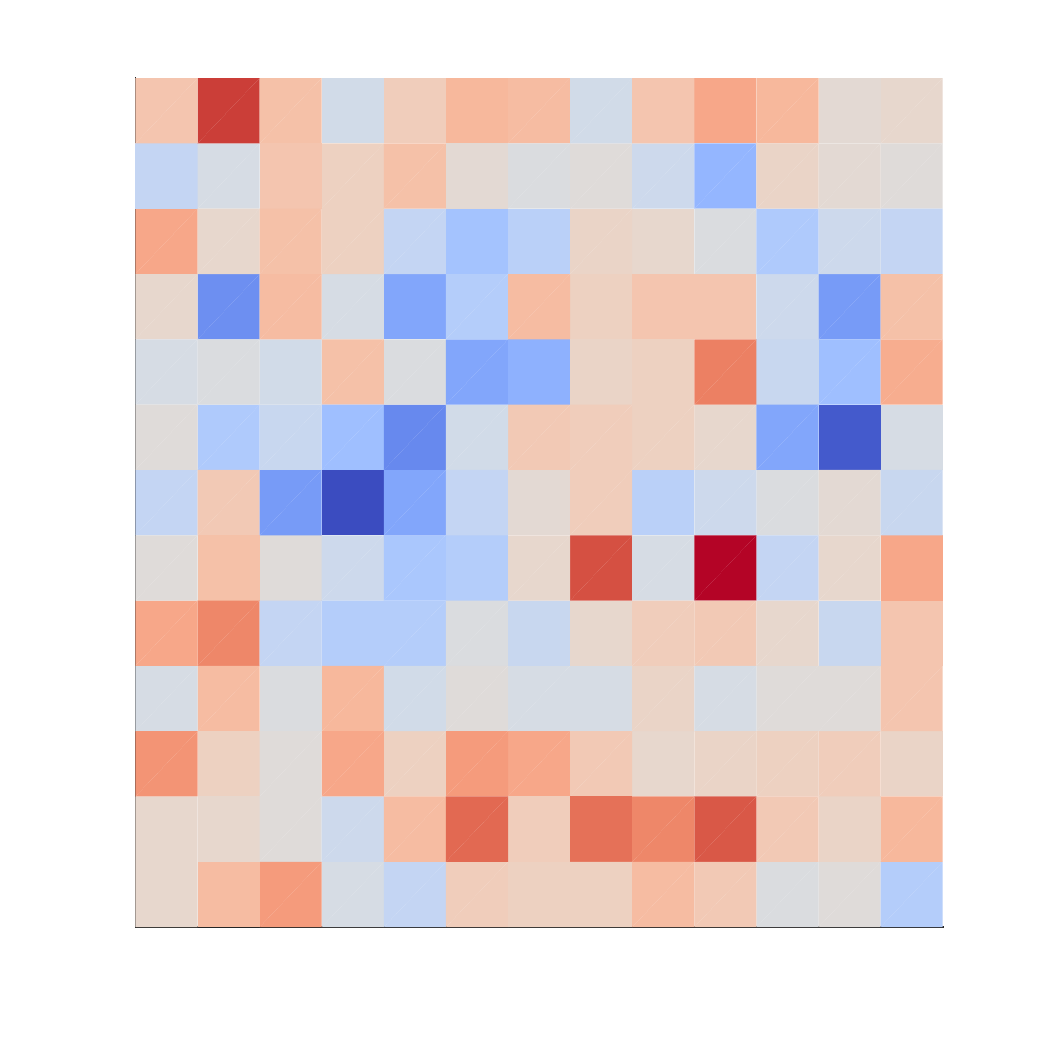}
	\includegraphics[clip = on, trim = 10mm 10mm 10mm 10mm ,width = 0.125\textwidth]{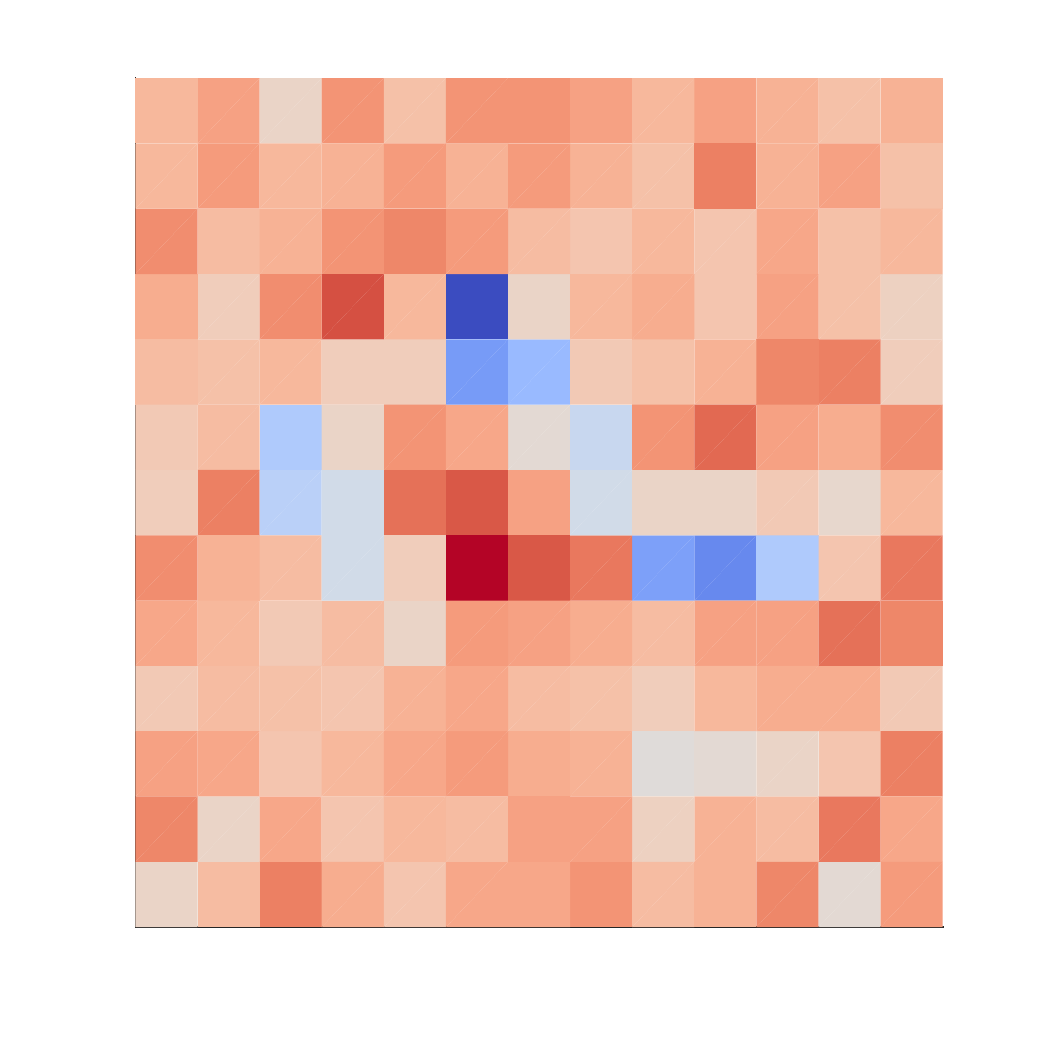}
	\includegraphics[clip = on, trim = 10mm 10mm 10mm 10mm ,width = 0.125\textwidth]{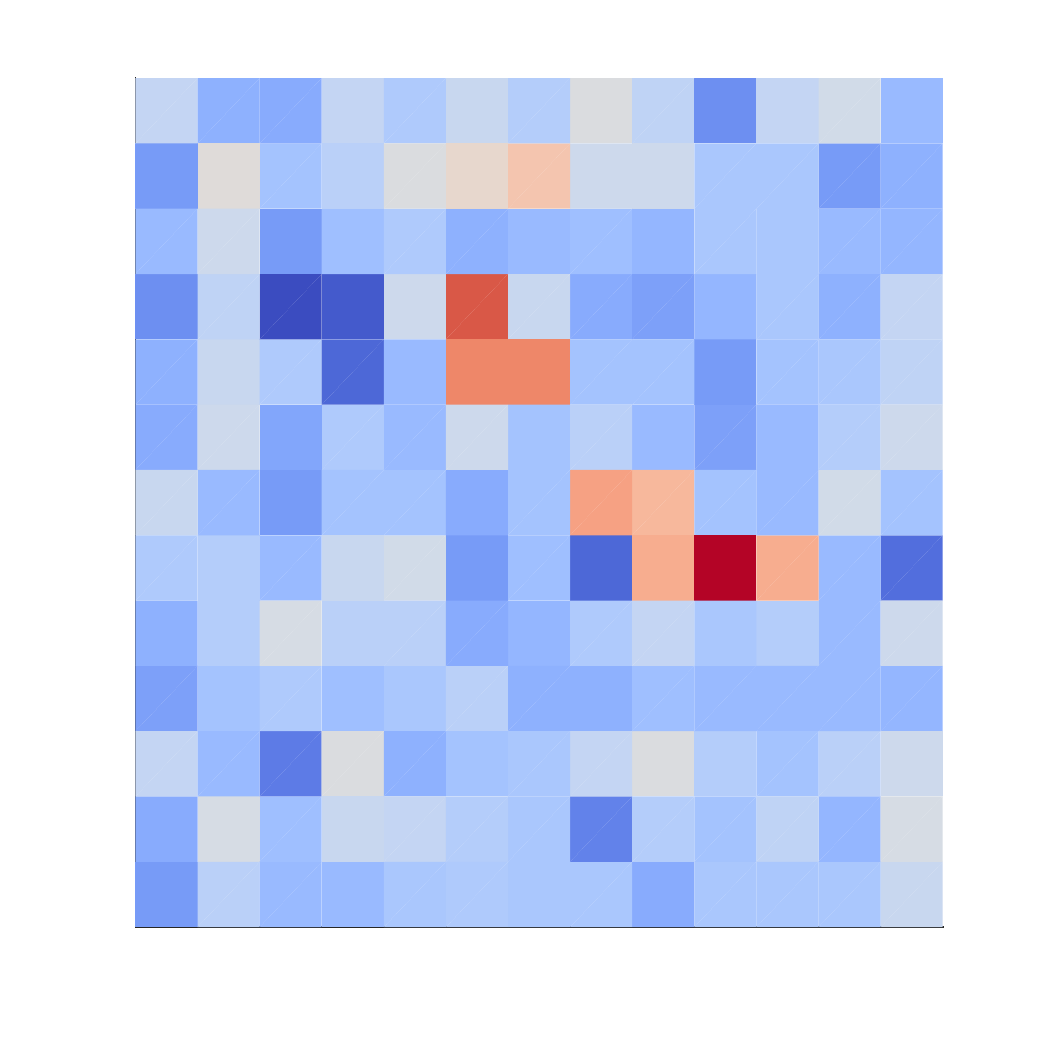} \\

	\includegraphics[clip = on, trim = 10mm 10mm 10mm 10mm ,width = 0.125\textwidth]{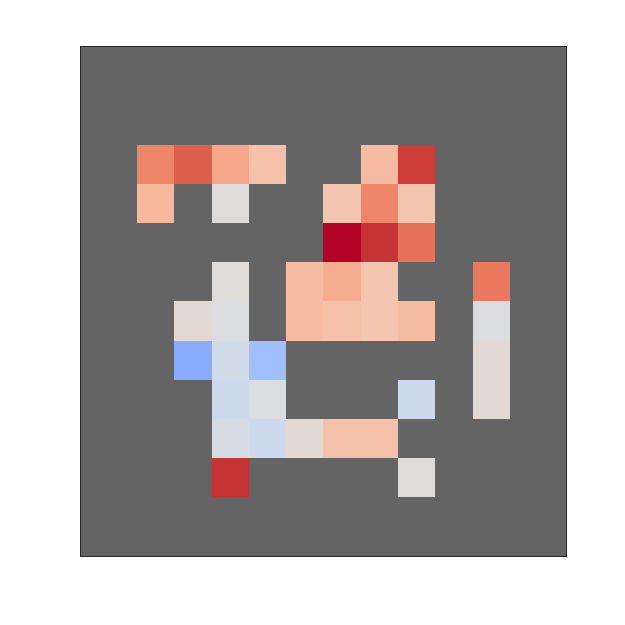}
	\includegraphics[clip = on, trim = 10mm 10mm 10mm 10mm ,width = 0.125\textwidth]{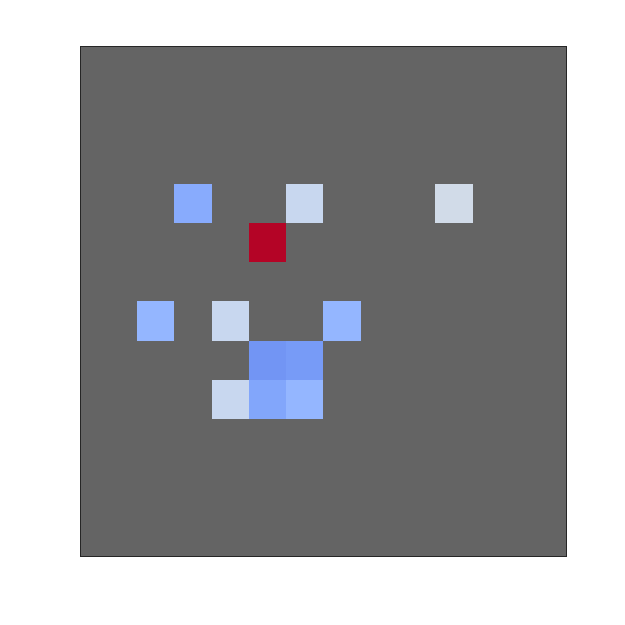}
	\includegraphics[clip = on, trim = 10mm 10mm 10mm 10mm ,width = 0.125\textwidth]{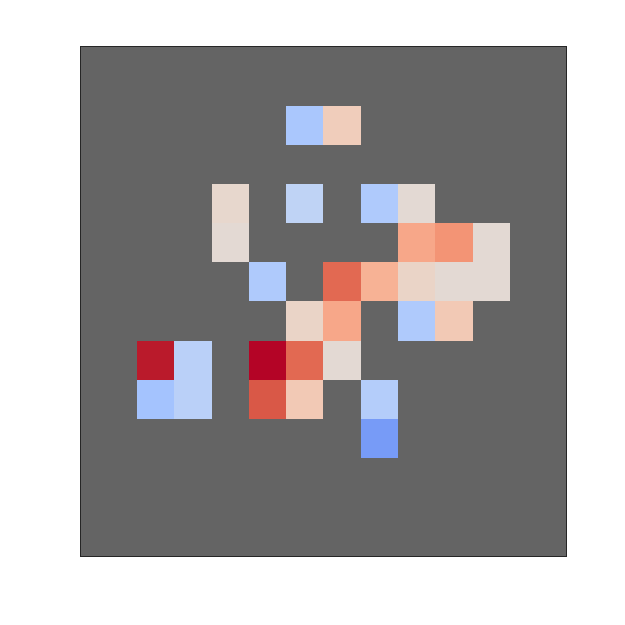}
	\caption{\textbf{First row}: Samples selected from MNIST358. \textbf{Second row}: Interpretability metric estimation using our method. \textbf{Third row}: Results obtained using LIME.} 
	\label{fig:Interpretability}
	\vspace*{-0.0cm}
\end{figure}
\paragraph{Local Interpretability for MNIST358}
Figure \ref{fig:Interpretability} shows the results for three samples selected from MNIST358 (top row), with the heat maps depicting the results of our method (second row) and those for LIME (third row, greyed out pixels are marked as irrelevant by LIME).
The colour gradient varies from red (positive impact, pixel value increase causing increased class probability of shown digit) to blue (negative impact, pixel value increase decreasing the class probability). 
For digit 3, our method obtains for example a contiguous blue patch on the left. Increasing the values of these pixels would modify the 3 into an 8. %, and the blue colour assigned by our method indicates that the corresponding decrease in class 3 probability would be reflected by the GPC model at hand. 
Indeed, when whitening the pixels of the blue patch, the class 3 probability assigned by the model decreases from $0.58$ to $0.40$.
Similarly, for digit 5, our methods identify a blue patch that would change the 5 into an 8 and again the GPC model indeed lowers its class 5 probability when the patch is whitened.
Similarly, for digit 8, our method identifies a blue patch of 3 pixels towards the top left, which would turn it into something resembling digit 3 if whitened. %This happens indeed to the classification probabilities, which increase from $0.19$ to $0.22$ for class $3$ while decreasing from $0.54$ to $0.5$ when adding $0.5$ to only these 3 pixels.
%The patterns found by LIME in these examples are either not as clear or different. 
%We argue that the different outcomes are due to the fact that, unlike LIME, our metric does not rely on the assumption of local linearity of the prediction model, which we inspect further in the next paragraph.
\begin{figure}[h]
	\centering
	{\hspace*{.5cm}  \includegraphics[width = 0.35\textwidth]{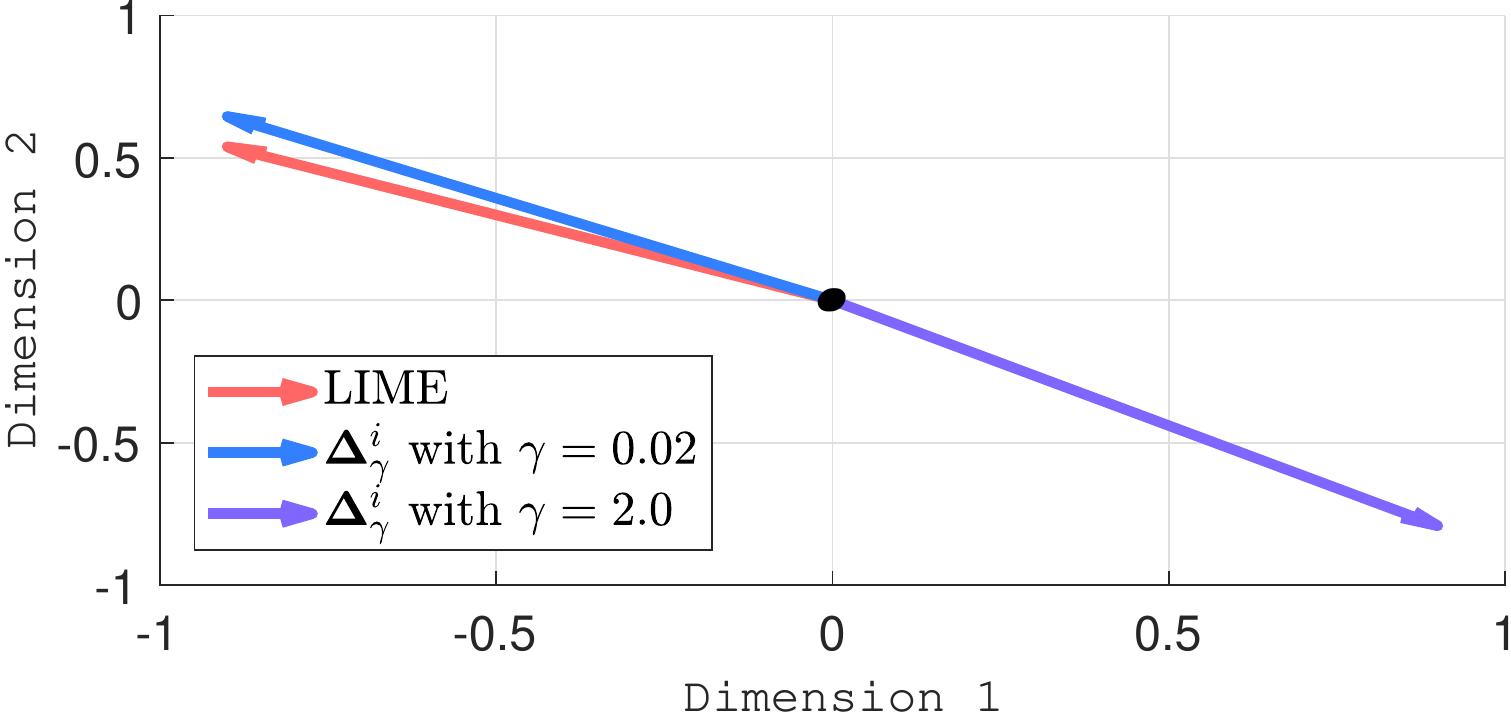}}
	{\hspace*{.5cm}  \includegraphics[width = 0.35\textwidth]{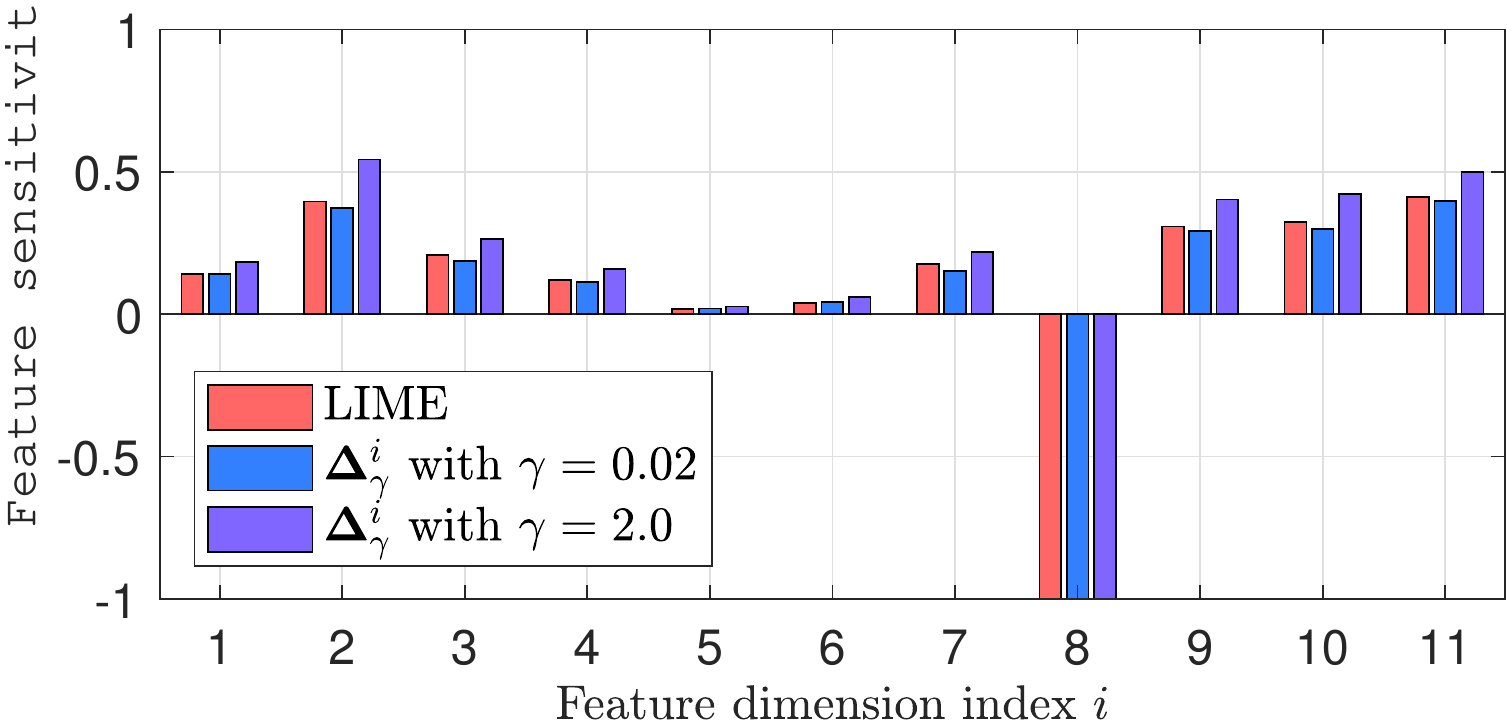}}
	%\vspace{0.2cm}
	\caption{Global feature sensitivity analysed by LIME and our metric $\mathbf{\Delta}^i_{\gamma}$. All values normed to unit scale for better comparison. \textbf{Top:} Results for Synthetic2D dataset mapped out on plane. \textbf{Bottom:} Results for SPAM dataset.} 
	\label{fig:interpret_appendix}
\end{figure}
\paragraph{Global Interpretability for the Binary Datasets}
We perform global interpretability analysis on the GPC models trained on the Synthetic2D and SPAM datasets, using $50$ random test points.
The results are shown in Figure \ref{fig:interpret_appendix}.
For Synthetic2D (top row), LIME suggests that a higher probability of belonging to class 1 (depicted as the direction of the arrow in the plot) corresponds to lower values along dimension 1 and higher values along dimension 2.
%For the synthetic 2D dataset (top row), we see that non-linearities of the GPC model can have a significant effect on the outcome of the interpretability analysis. LIME's global metric, which averages over the local metrics for each test point, indicates that a higher probability of belonging to Class 1 (direction of the arrow) corresponds to lower values along Dimension 1 and higher values along Dimension 2.
As can be seen in the corresponding contour plot in Figure \ref{fig:safety1} (top left), the exact opposite is true however.
LIME, being built on linearity approximations, fails to take into account the global behaviour of the GPC.
%Inspecting the contour lines and test point locations, it is easy to understand why: most test points lie in a location where, in the immediate neighbourhood, this global relationship is reversed.
%Aggregating the local effects thus gives an incorrect global impression, which is also true for a very small $\gamma$ corresponding to a very local application of our metric $\mathbf{\Delta^i_{\gamma}}$. However our metric is able to overcome this shortfall by using a larger $\gamma$, for which the global relationship is correctly reflected.
When using a small value of $\gamma$ our approach obtains similar results to LIME.
However, with $\gamma = 2.0$ the global relationship between input and output values is correctly captured.
For SPAM, on the other hand (Figure \ref{fig:interpret_appendix}, bottom), due to linearity of the dataset and GPC, a local analysis correctly reflects the global picture.
%The results computed in fact are similar both for LIME and the method we compute, with slight quantitative differences with $\gamma = 2.0$.
%In this case, LIME is probably a better choice for model interpretation as it takes less time to calculate its metrics.

%\rev{Altogether, these results demonstrate that our method can offer additional non-linear insights that might be missed by existing state-of-the-art methods like LIME \citep{ribeiro2016should}. While we do not claim that it is superior to LIME (or other methods), we believe it is complementary and offers benefits when performing GPC model selection.}

%
\section{CONCLUSION}
We presented a method for computing, for any compact set of input points, the class probability range of a GPC model across all points in that set, up to any precision $\epsilon > 0$. This allows us to analyse robustness and safety against adversarial attacks, which we have demonstrated on multiple datasets and approximate Bayesian inference techniques.
%As our theoretical results are valid for any Gaussian approximation, we plan to use them to formally quantify the robustness of other Bayesian classification models. %, such as Bayesian neural networks.
%\subsubsection*{Acknowledgements}
%Use unnumbered third level headings for the acknowledgements title.
%All acknowledgements go at the end of the paper.
\section{ACKNOWLEDGMENT}
This project was partly funded by the EU's Horizon 2020 research and innovation programme under the Marie Sk\l{}odowska-Curie (grant agreement No 722022), the ERC under the European Union’s Horizon 2020 research and innovation programme (grant agreement No. 834115) and the EPSRC Programme Grant on Mobile Autonomy (EP/M019918/1). Further funding was received from the Konrad-Adenauer-Stiftung and the Oxford Man-Institute of Quantitative Finance.

{
%\small
\bibliographystyle{iclr2020_conference}
\bibliography{Biblio}

}

\clearpage
\appendix
\section*{Supplementary Material: ADVERSARIAL ROBUSTNESS GUARANTEES FOR CLASSIFICATION WITH GAUSSIAN PROCESSES}
In the first Section of this Supplementary Material we present the proof of Propositions \ref{Theorem:BOundsGenericBIclass} and \ref{Prop:Gaussian}, as well as Theorem \ref{TH:convergenceBrenchAndBOund}.
Further technical results concerning multiclass classification are treated in Section \ref{app:multiclass}.
In Section \ref{app:probit} we detail the case of binary classification using the probit likelihood function.
In Section \ref{app:variancebound} we detail our approach for computing a lower bound of the predictive variance and mention how promising candidate points for the GPC bounding can be computed. 
We empirically analyse the computational complexity of the branch and bound methodology in a runtime analysis in Section \ref{app:runtime}.
In Section \ref{app:experiments} we describe the datsets used and detail the experimental settings. 
Finally, in Section \ref{app:interpretability}, details for the interpretability metric we use in the experimental section are given.

\section{PROOFS FOR BINARY CLASSIFICATION BOUNDS}
\label{app:proofs}
\subsection{Proof of Proposition \ref{Theorem:BOundsGenericBIclass}}

\begin{proof}
We detail the proof for $\min_{x \in T} \pi(x \vert \mathcal{D}).$ The $\max$ case follows similarly.
\begin{align*}
  &\min_{x \in T} \pi(x \vert \mathcal{D})  \\
  &\quad \text{(By definition)}\\
  =&\min_{x \in T}\int_{- \infty}^{+ \infty} \sigma (\bar{f}) q(f(x)=\bar{f} | \mathcal{D})d\bar{f}\\
  &\quad \text{(By additivity of integrals)}\\
   =&  \min_{x \in T}\sum_{i=1}^{N} \int_{a_i}^{b_i} \sigma(\bar f) q(f(x)=\bar{f} | \mathcal{D})d\bar{f} \\
   &\quad \text{(By monotonicity of $\sigma$ and non-negativity of $q$)}\\
     \geq & \min_{x \in T }\sum_{i=1}^{N}   \int_{a_i}^{b_i} \sigma(a_i) q(f(x)=\bar{f} | \mathcal{D})d\bar{f}  \\
     &\quad \text{(By definition of minimum and of $q$)} \\
    \geq &  \sum_{i=1}^{N}  \sigma(a_i) \min_{x \in T} \int_{a_i}^{b_i} \mathcal{N}(\bar f | \mu(x),\Sigma(x))d \bar{f}
\end{align*}
\end{proof}

\subsection{Proof of Proposition 2}
\begin{proof}
We provide the proof for the $\min$ case, similar arguments hold for the $\max$.
By definition of $\mu^L_T$, $\mu^U_T$, ${\Sigma}^L_T $, ${\Sigma}^U_T$ we have that:

\begin{align*}
&\min_{x \in T} \int_{a}^{b} \mathcal{N}(\bar{f} | \mu(x),\Sigma(x))d \bar{f} 
\geq\\
&\min_{\substack{  \mu \in [\mu^L_T,\mu^U_T] \\ \Sigma \in [{\Sigma}^L_T ,{\Sigma}^U_T]  }} \int_{a}^{b} \mathcal{N}(\bar{f} | \mu,\Sigma) d \bar{f} =\\
&\frac{1}{2} \min_{\substack{  \mu \in [\mu^L_T,\mu^U_T] \\ \Sigma \in [{\Sigma}^L_T ,{\Sigma}^U_T]  }} \left( \erf \left(\frac{\mu-a}{ \sqrt{2\Sigma}}\right) - \erf \left(\frac{\mu-b}{ \sqrt{2\Sigma}}\right) \right):= \\
&\frac{1}{2} \min_{\substack{  \mu \in [\mu^L_T,\mu^U_T] \\ \Sigma \in [{\Sigma}^L_T ,{\Sigma}^U_T]  }} \Phi(\mu,\Sigma) . 
\end{align*}
By looking at the partial derivatives we have that:
\begin{align*}
    &\frac{\partial \Phi(\mu,\Sigma)}{\partial \mu} =\\
    &\frac{\sqrt{2}}{\sqrt{ \pi \Sigma }} \left( e^{-\frac{({\mu} - b)^2}{2\Sigma}  } -  e^{-\frac{({\mu} - a)^2}{2\Sigma}  } \right) 
    \geq 0 \Leftrightarrow 
     \mu \leq \frac{a + b}{2} =: \mu^m
\end{align*}
and that if $\mu \not\in [a,b]$: 
\begin{align*}
    &\frac{\partial \Phi(\mu,\Sigma)}{\partial \Sigma} = \\
&    \frac{1}{\sqrt{2 \pi \Sigma^3 }} \left( ({\mu} - b_i) e^{-\frac{({\mu} - b_i)^2}{2\Sigma^2}  } - ({\mu} - a_i) e^{-\frac{({\mu} - a_i)^2}{2\Sigma^2}  } \right) 
    \geq 0\\ 
    &\Leftrightarrow 
    \Sigma \leq \frac{{(\mu-a)^2-(\mu-b)^2}}{{2 \log \frac{\mu-a}{\mu-b}}} := \Sigma^m(\mu) 
\end{align*}
otherwise the last inequality has no solutions.
As such $\mu^m$ and $\Sigma^m$ will correspond to global maximum wrt to $\mu$ and $\Sigma$ respectively.
As $\Phi$ is symmetric wrt $\mu^m$ we have that the minimum value wrt to $\mu$ is always obtained for the point furthest away from $\mu^c$, that is: $\underline{\mu} = \argmax_{\mu \in [ \mu^L_T, \mu^U_T ]} \vert \mu^m - \mu \vert $.
The minimum value wrt to $\Sigma$ will hence be either for $\Sigma^L_T$ or $\Sigma^U_T$, that is $\underline{\Sigma} =\argmin_{\Sigma \in \{\Sigma^L_T, \Sigma^U_T\}} \Phi(\underline{\mu},\Sigma).$

\end{proof}

\subsection{Proof of Theorem \ref{TH:convergenceBrenchAndBOund} }

\begin{proof}
We consider the $\min$ case. The $\max$ case follows similarly.

In order to show the convergence of the branch and bound, we need to show that for any test point $x$ there exists $r>0$ and a partition of the latent space $\mathcal{S}=\{S_i,i=\{1,...,N \} \}$ such that for the  interval $I=[x-rI,x+rI]$ we have that for any $\bar x \in I$
$$| \pi(\bar x \vert \mathcal{D} )-\sum_{i=1}^{N}  \sigma(a_i) \min_{x \in I} \int_{a_i}^{b_i} \mathcal{N}(\bar f | \mu(x),\Sigma(x))d \bar{f}|\leq \epsilon.$$
In order to do that, we first observe that by the Lipschitz continuity of mean and variance we have that for $x_1,x_2 \in I$, it holds that
$$ |\mu(x_1)- \mu(x_2)|\leq K^{\mu}r  $$
$$ |\Sigma(x_1)- \Sigma(x_2)|\leq K^{\Sigma}r, $$
for certain $K^{\mu},K^{\Sigma}>0.$
Now, for $S_i \in \mathcal{S},$ consider  $x^i$ such that $ \int_{a_i}^{b_i} \mathcal{N}(\bar f | \mu(x^i),\Sigma(x^i))d \bar{f}=\min_{x \in I} \int_{a_i}^{b_i} \mathcal{N}(\bar f | \mu(x),\Sigma(x))d \bar{f}$.
Further, due to the monotonicity and continuity  of $\sigma$, we can consider a uniform discretisation of the y-axis for $\sigma$ in $N$ intervals. That is, for all $S_i\in \mathcal{S}$, we have that $\sigma(b_i)=\sigma(a_i)+\frac{1}{N}$.
At this point, for any $\bar x \in I$ the following calculations follow
\begin{align}
    & |\pi(\bar{x} \vert \mathcal{D} )-\sum_{i=1}^{N}  \sigma(a_i)  \int_{a_i}^{b_i} \mathcal{N}(\bar f | \mu(x^i),\Sigma(x^i))d \bar{f} |\\  
    & \text{(By Definition)} \nonumber\\
    =&\nonumber \vert  \int  \sigma(\bar f)  \mathcal{N}(\bar f | \mu(\bar x),\Sigma(\bar x))d \bar{f}-\\
    &\quad \quad  \quad  \sum_{i=1}^{N}  \sigma(a_i)  \int_{a_i}^{b_i} \mathcal{N}(\bar f | \mu(x^i),\Sigma(x^i))d \bar{f} \vert\\
    &\text{(By additivity of integral and re-ordering terms)}  \nonumber\\
    =\nonumber&|\sum_{i=1}^{N} \big ( \int_{a_i}^{b_i} \sigma(\bar f)  \mathcal{N}(\bar f | \mu(\bar x),\Sigma(\bar x))d \bar{f} \, - \,\\
    &\quad \quad \quad \int_{a_i}^{b_i} \sigma(a_i) \mathcal{N}(\bar f | \mu(x^i),\Sigma(x^i))d \bar{f} \big)|\\
    & \text{(As for any $\bar{f}\in S_i$, 
    $\sigma(a_i)\leq \sigma(f_i)\leq \sigma(a_i)+\frac{1}{N} $ })\nonumber\\
%    $\int_{a_i}^{b_i} \sigma(\bar f)  \mathcal{N}(\bar f | \mu(\bar x),\Sigma(\bar x))d \bar{f} \geq \int_{a_i}^{b_i} \sigma(a_i) \mathcal{N}(\bar f | \mu(x^i),\Sigma(x^i))d \bar{f}$ )}  \nonumber\\
  \nonumber  \leq & | \sum_{i=1}^{N} \big ( \int_{a_i}^{b_i} (\sigma(a_i) + \frac{1}{N} )  \mathcal{N}(\bar f | \mu(\bar x),\Sigma(\bar x))  \, - \, \noindent\\
& \quad  \quad \quad \quad \quad \quad  \quad \quad \quad  \quad   \sigma(a_i) \mathcal{N}(\bar f | \mu(x^i),\Sigma(x^i))d \bar{f} \big) |\\ 
    & \text{(By Triangle Inequality)}  \nonumber\\ \nonumber
   \leq & | \sum_{i=1}^{N}  \int_{a_i}^{b_i}  \frac{1}{N}   \mathcal{N}(\bar f | \mu(\bar x),\Sigma(\bar x)) d\bar f | + \\ \nonumber
   &| \sum_{i=1}^{N} \big( \sigma(a_i)  \int_{a_i}^{b_i}  \mathcal{N}(\bar f | \mu(\bar x),\Sigma(\bar x))   - \\
   &\quad  \quad \quad \quad \quad \quad  \quad \quad \quad  \quad   \mathcal{N}(\bar f  \mu(x^i),\Sigma(x^i))d \bar{f} \big) | \\
    &\text{(By Re-ordering terms and Triangle Inequality)}  \nonumber\\
    \leq & | \frac{1}{N} \int     \mathcal{N}(\bar f | \mu(\bar x),\Sigma(\bar x)) d\bar f | \nonumber\\  
    & \nonumber + \sum_{i=1}^{N} \sigma(a_i) |  \int_{a_i}^{b_i}  \mathcal{N}(\bar f | \mu(\bar x),\Sigma(\bar x))  \, - \\
    & \quad  \quad \quad \quad \quad \quad  \quad \quad \quad  \quad  \mathcal{N}(\bar f \mu(x^i),\Sigma(x^i))d \bar{f}  | \\
    &\text{(By properties of integrals and  $\sigma(f)\in [0,1]$)}   \nonumber\\ \nonumber
    \leq &\frac{1}{N} + \sum_{i=1}^{N} | \int_{a_i}^{b_i} ( \mathcal{N}(\bar f | \mu(\bar x),\Sigma(\bar x))  \, -\\
    &\quad  \quad \quad \quad \quad \quad  \quad \quad \quad  \quad  \ \mathcal{N}(\bar f | \mu(x^i),\Sigma(x^i)) ) d \bar{f} |
    \label{Eqn:FinalFOrmProof}
\end{align}
Now, as $ |\mu(\bar x),- \mu(x^i),|\leq K^{\mu} r  $ and
$ |\Sigma^2(\bar x)- \Sigma^2(x^i)|\leq K^{\Sigma} r$, we have that as $r\to 0$ both mean and variance converge to the same value. Hence, this implies that for each $S_i \in S$
\begin{align*}\lim_{r \to 0} \big( \int_{a_i}^{b_i}  \mathcal{N}&(\bar f | \mu(\bar x),\Sigma(\bar x)) d \bar f \, - \,\\ &\int_{a_i}^{b_i} \mathcal{N}(\bar f | \mu(x^i),\Sigma(x^i))d \bar{f} \big) =0.
\end{align*}
 As a consequence, for any  $\epsilon>0$, we can choose $N=\lceil \frac{2}{\epsilon} \rceil$ and then select $r$ such that the second term in Eqn \eqref{Eqn:FinalFOrmProof} is bounded by $\frac{\epsilon}{2}.$

\end{proof}

\section{BOUNDS FOR MULTICALSS CLASSIFICATION}
\label{app:multiclass}

\subsection*{Proof of Proposition \ref{Theorem:BOundsGenericMulticlass}}
\begin{proof}
We detail the proof for $\min_{x \in T} \pi^c(x \vert \mathcal{D} ).$ The $\max$ case follows similarly.
\begin{align*}
  &\min_{x \in T} \pi^c(x \vert \mathcal{D} )\\
  &\quad \text{(By definition)}\\
  &=\min_{x \in T}\int \sigma^c (\bar{f}) q(f(x)=\bar{f} | \mathcal{D})d\bar{f}\\
  &\quad \text{(By additivity of integral)}\\
   &=  \min_{x \in T}\sum_{i=1}^{N} \int_{S_i} \sigma^c(\bar f) q(f(x)=\bar{f} | \mathcal{D})d\bar{f}  \\
  &    \quad \text{(Because $q$ is non-negative)}\\
  & \geq  \min_{x \in T}\sum_{i=1}^{N} \int_{S_i} \min_{y \in S_i}\sigma^c(y) q(f(x)=\bar{f} | \mathcal{D})d\bar{f} \\
&    \quad \text{(By definition of infimum)} \\
    & \geq  \sum_{i=1}^{N} \min_{y \in S_i}\sigma^c(y) \min_{x \in T} \int_{S_i}  q(f(x)=\bar{f} | \mathcal{D})d\bar{f} \\
&    \quad \text{(By Definition of $q$)}\\
     &=   \sum_{i=1}^{N} \min_{y \in S_i}\sigma^c(y) \min_{x \in T} \int_{S_i} \mathcal{N}(\bar{f} | \mu(x),\Sigma(x))d \bar{f} 
\end{align*}
\end{proof}

Proposition \ref{Theorem:BOundsGenericMulticlass} in the main text implies that if we can compute infimum and supremum of the softmax over a set of the latent space (shown in Lemma \ref{lemma:computationMinMaxMulticlass}) and the mean and covariance matrix that maximise a Gaussian integral (shown in Proposition \ref{Prop:MultiCLass}), then upper and lower bounds on $\piInf{T}$ and $\piSup{T}$ can be derived. 

\begin{lemma}
\label{lemma:computationMinMaxMulticlass}
Let $S\subset \mathbb{R}^{|C|}$  be an axis-parallel hyper-rectangle.   Call $f^{max}=\argmax_{f \in S}\sigma^c(f)$ and $f^{min}=\argmin_{f \in S}\sigma^c(f)$. Assume $\sigma$ is the softmax function. Then, $f^{max}$ and $f^{min}$ are vertices of $S.$
\end{lemma}

\begin{proof}
$S$ is an axis-parallel hyper-rectangle. As a consequence, it can be written as intersection of constraints of the form
$- f_i\leq -k_{i,1}$ and $ f_i\leq k_{i,2}$, where $f_i$ is the i-th component of vector $f$.
Hence, the optimisation problem for the maximisation case (minimisation case is equivalent) can be rewritten as follows:
\begin{align*}
    &\max  \sigma^c(f)\\
    \text{such that }&\forall i \in \{1,...,|C| \}-f_i\leq -k_{i,1},\quad f_i\leq k_{i,2}. 
\end{align*}
In order to solve this problem we can apply the  Karush-Kuhn-Tucker (KKT) conditions.
Being the constraints independent of $f$, the KKT conditions imply that in order to conclude the proof we just need to show that for all $f \in S, c \in \{1,...,|C| \}$, $\frac{ d \sigma^c(f)}{d f_c}\neq 0.$ This is shown in what follows.

For $f \in \mathbb{R}^n$ and $c \in \{1,...n \}$ We have 
$$ \sigma^c(f)=\frac{\exp ( f_c)}{\sum_{j=1}^C \exp (f_{j})}.$$
Then, we obtain
$$ \frac{d \sigma^c(f)}{d f_c}=\frac{\exp({f_c}) (\sum_{j\neq c} \exp (f_{j}))}{(\sum_{j=1}^C \exp (f_{j}))^2}, $$
while for $i \neq c$ we have
$$ \frac{d \sigma^c(f)}{f_i}=-\frac{\exp({f_c}) \exp (f_{i})}{(\sum_{j=1}^C \exp (f_{j}))^2}. $$
This implies that for $f \in \mathbb{R}^n$ and $i\neq c$ we always have
$$ \frac{d \sigma^c(f)}{d f_c}>0  \quad \quad   \frac{d \sigma^c(f)}{d f_i}<0  .$$
\end{proof}

Note that in Lemma \ref{lemma:computationMinMaxMulticlass} we assumed that $S$ is an hyper-rectangle. However, the lemma can be trivially extended to more general sets given by the intersection of arbitrarily many  half-spaces generated by hyper-planes perpendicular to one of  the axis. 

The following Lemma is needed to prove Proposition \ref{Prop:MultiCLass}.

\begin{lemma}
\label{Lemma:supProbab}
Let $X$ and $Y$ be random variables with joint density function $f$. Consider measurable sets $A$ and $B$. Then, it holds that 
$$ P(X\in A | Y\in B)\leq \sup_{y \in B}  P(X\in A | Y=y). $$
\end{lemma}
\begin{proof}
\begin{align*}
    &P(X\in A | Y \in B)\\
    & = \frac{P( X\in A \wedge Y \in B  )}{P(Y \in B)}\\
    &= \frac{\int_{x \in A} \int_{y \in B}  f(X=x \wedge Y= y) dx dy  }{P(Y \in B)}\\
    &=  \frac{\int_{x \in A} \int_{y \in B}  f(X=x | Y= y) f(Y=y) dx dy  }{P(Y \in B)}\\
    &\leq   \frac{\int_{x \in A} \int_{y \in B} \sup_{\bar y \in B } f(X=x | Y= \bar y) f(Y=y) dx dy  }{P(Y \in B)}\\
    &=\frac{\int_{x \in A} \sup_{\bar y \in B } f(X=x | Y= \bar y) dx \int_{y \in B}  f(Y=y)dy  }{P(Y \in B)}\\
    &=\frac{ \sup_{ y \in B} P(X \in A | Y=y)  P(X \in B)  }{P(Y \in B)}\\
    &=  \sup_{ y \in B} P(X \in A | Y=y),
\end{align*}
\end{proof}

\subsection{Proof of Proposition \ref{Prop:MultiCLass}.}
We consider the supremum case. The infimum follows similarly.
Let $\mathbf{y}(x)$ be a normal random variable with mean $\mu(x)$ and covariance matrix $\Sigma(x).$  Then, we have 
\begin{align*}
&\sup_{x\in T}\int_{S}\mathcal{N}(\bar{f}|\mu(x),\Sigma(x)) d \bar{f} \\
=&\sup_{x \in T} {P(\mathbf{y}(x)\in S)}\\
=&\sup_{x \in T}  P(\wedge_{i=1}^{C}k_{i}^1\leq \mathbf{y}_i(x)\leq k_{i}^{2})\\
=&\sup_{x \in T} \prod_{i=1}^{C} {P(k_{i}^1\leq \mathbf{y}_i(x)\leq k_{i}^{2}| \wedge_{j=i+1}^C k_{j}^1\leq \mathbf{y}_j(x)\leq k_{j}^2 ) }\\
&\quad \text{(By Lemma \ref{Lemma:supProbab})}\\
\leq & \sup_{x \in T} \prod_{i=1}^{C} \sup_{f \in S^{i+1}} P(k_{i}^1\leq \mathbf{y}_i(x)\leq k_{i,2}| \\
&  \quad \quad \quad \quad \quad \quad \quad \quad \quad \quad \quad \quad \wedge_{j=i+1}^C \mathbf{y}_j(x)=f_{j-i})  \\
\leq &  \prod_{i=1}^{C}\sup_{x \in T,f \in S^{i+1}} P(k_{i}^1\leq \mathbf{y}_i(x)\leq k_{i}^2| \\
& \quad \quad \quad \quad \quad \quad \quad \quad \quad \quad \quad \quad \wedge_{j=i+1}^C \mathbf{y}_j(x)=f_{j-i})  \\
\end{align*}
Notice that for each $i \in \{1,...,C \}$, $P(k_{i}^1\leq \mathbf{y}_i(x)\leq k_{i}^2| \wedge_{j=i+1}^C \mathbf{y}_j(x)=f_{j-i})$ is the integral of a uni-dimensional Gaussian random variable, as a Gaussian random variable conditioned to a jointly Gaussian random variable is still Gaussian.

\section{BOUNDS FOR PROBIT BINARY CLASSIFICATION}
\label{app:probit}
For the case that the likelihood $\sigma$ is taken to be the probit function, that is, $\sigma (\bar{f})= \Phi ( \lambda \bar{f})$ is the cdf of the univariate standard Gaussian distribution scaled by $\lambda > 0$, it holds that  
\begin{equation*}
    \pi (x\, | \, \mathcal{D}) = \Phi \left( \frac{\mu (x)}{ \sqrt{ \lambda^{-2} + \Sigma (x) }} \right),
\end{equation*} 
where $\mu (x)$ and $\Sigma (x)$are the mean and variance of  $q(f(x)=\bar f | \mathcal{D})$ \cite{bishop2006pattern}. We can use this result to derive analytic upper and lower bounds for Eqn \eqref{Eq:infSubROsustness} without the need to apply Proposition \ref{Theorem:BOundsGenericBIclass}, by relying on upper and lower bounds for the latent mean and variance functions. 
This can be obtained by direct inspection of the derivatives of $\pi(x \vert \mathcal{D})$.
\begin{lemma}\label{lemma:bounds_probit}
Let $T \subseteq \mathbb{R}^d$. Then, we have that
\begin{align} \label{eq:lb_probit1}
    \Phi \left( \frac{\mu_T^L}{\sqrt{\lambda^{-2} + \underline{\Sigma}}} \right) \leq \piInf{T} \quad
    \end{align}
    and
    \begin{align} \label{eq:lb_probit2}
    \piSup{T} \leq \Phi \left( \frac{\mu_T^U}{\sqrt{\lambda^{-2} + \overline{\Sigma}}} \right) 
\end{align}
with $\underline{\Sigma}={\Sigma}^U_T$ if $\mu_T^L \geq 0$ and ${\Sigma}^L_T $ otherwise, while $\overline{\Sigma}={\Sigma}^L_T $ if $\mu_T^U \geq 0$ and ${\Sigma}^U_T$ otherwise.
\end{lemma}

\section{BOUNDS ON LATENT MEAN AND VARIANCE}
\label{app:variancebound}
In this section of the Supplementary Material we briefly review how lower and upper bounds on the a-posteriori mean and variance can be computed, and further show how this give us candidate points for the evaluation of bounds (that is line 6 in Algorithm \ref{alg:bnb_sketch} of the main paper).

We obtain bounds on latent mean and variance by applying the framework presented in \cite{cardelli2018robustness} for computation of $\mu^L_T, \mu^U_T$ and $\Sigma^U_T $, and subsequently extend it for the computation of $\Sigma^L_T$.
Briefly, assuming continuity and differentiability of the kernel function defining the GPC covariance, it is possible to find linear upper and lower bounds on the covariance vector, which can be propagated through the inference formula for $q(f(x)=\bar{f} | \mathcal{D})$.
The bounding functions obtained in this way can be analytically optimised for $\mu^L_T$ and $\mu^U_T$, while convex quadratic programming is used to obtain $\Sigma^U_T$ (see \cite{cardelli2018robustness} for details).
Finally, we solve the concave quadratic problem that arises when computing
$\Sigma^L_T$ by adapting methods introduced in \cite{rosen1986global}, which reduces the problem to the solution of $2 \vert \mathcal{D} \vert + 1$ linear programming problems.
This is detailed in the following subsection.

As discussed in Section \ref{Sec:BinaryBounds} in order to obtain $\piInfU{T}$ and $\piSupL{T}$ it suffice to evaluate the GPC in any point inside $T$.
However, the more close $\piInfU{T}$ and $\piSupL{T}$ are to $\piInf{T}$ and $\piSup{T}$ respectively, the more quicker will be the convergence of the branch and bound algorithm (as per line 7 in Algorithm \ref{alg:bnb_sketch} in the main paper).
Notice that, in solving the optimisation problems associated to  $\mu^L_T, \mu^U_T$, $\Sigma^U_T $ and $\Sigma^L_T$ we obtain four extrema points in $T$ on which the GPC assume the optimal values a-posteriori mean and variance values.
As these points belong to $T$ and provide extreme points for the latent function they make promising candidates for the evaluation of $\piInfU{T}$ and $\piSupL{T}$.
Specifically in line 6 of Algorithm \ref{alg:bnb_sketch} (main paper), we evaluate the GPC on all four the extrema and select the one that gives the best bound among them. 
\subsection{Lower Bound on Latent Variance}
Let $\covvec(x) = [r_1(x),\ldots,r_M(x)]$ be the vector of covariance between a test point and the training set $\mathcal{D}$ with $ \vert \mathcal{D} \vert = M$, and let $R$ be the inverse covariance matrix computed in the training set, and $\Sigma_p$ be the (input independent) self kernel value.
By explicitly using the variance inference formula, we are interested in finding a lower bound for: $\min_{x \in T} \left( \Sigma_p - \covvec(x)^T R  \covvec(x) \right) = \Sigma_p + \min_{x \in T} \left( - \covvec(x)^T R \covvec(x)  \right)$.
We proceed by introducing the $M$ auxiliary variables $r_i = \covvec_i(x)$, yielding a quadratic objective function on the auxiliary variable vector $\covvec = [r_1, \ldots, r_M]$, that is $ - \covvec^T R \covvec $. 
Analogously to what is done in \cite{cardelli2018robustness} we can compute two matrices $A_r$, $A_x$ and a vector $b$ such that $\covvec = \covvec(x)$ implies $A_r \covvec + A_x x \leq b $, hence obtaining the quadratic program:
\begin{align}\label{prob:concave}
    \min - \covvec^T R \covvec \\
    \textrm{Subject to:} \quad &A_r \covvec + A_x x \leq b \nonumber \\
    &r_i^L \leq r_i \leq r_i^U \qquad  i = 1,\ldots, M  \nonumber  \\
    &x_i^L \leq x_i \leq x_i^U \qquad  i = 1,\ldots, m  \nonumber  
\end{align}
whose solution provides a lower bound (and hence a safe approximation) to the original problem $\min_{x \in T} \left( - \covvec(x)^T R \covvec(x)  \right)$.
Unfortunately, as $R$ is positive definite, we have that $-R$ is negative definite; hence the problem posed is a concave quadratic program for which a number of local optima may exist.
As we are instead dealing with worst-case scenario analyses, we are  actually interested in computing the global minimum. 
This however is an NP-hard problem \cite{rosen1986global} whose exact solution would make a branch and bound algorithm based on it impractical.
Following the methods discussed in \cite{rosen1986global}, we instead proceed to compute a safe lower bound to that.
The main observation is that,  being $R$ symmetric positive definite, there exist a matrix of eigenvectors $U = [ \mathbf{u}_1, \ldots ,\mathbf{u}_M ] $ and a diagonal matrix of the associated eigenvalues $\lambda_i$ for $i = 1,\ldots,M$, $\Lambda $ such that $R = U \Lambda U^T$.
We hence define $\hat{r}_i =  \mathbf{u}_1^T r_i$ for $i = 1,\ldots,M$ to be the rotated variables and compute their ranges $[\hat{r}_i^L,\hat{r}_i^U]$ by solution of the following $2M$ linear programming problems:
\begin{align*}
    \min / \max \quad  &\mathbf{u}_i^T r_i \\
    \textrm{Subject to:} \quad \quad &A_r \covvec + A_x x \leq b \\
    &r_j^L \leq r_i \leq r_j^U \qquad  j = 1,\ldots, M \\
    &x_j^L \leq x_i \leq x_j^U \qquad  j = 1,\ldots, m. 
\end{align*}

Implementing the change of variables into Problem \ref{prob:concave} we obtain:
\begin{align*}
    \min - \mathbf{\hat{r}}^T \Lambda \mathbf{\hat{r}} \\
    \textrm{Subject to:} \quad &\hat{A}_{\hat{r}} \mathbf{\hat{r}} + A_x x \leq b  \\
    &\hat{r}_i^L \leq \hat{r}_i \leq \hat{r}_i^U \qquad  i = 1,\ldots, M    \\
    &x_i^L \leq x_i \leq x_i^U \qquad  i = 1,\ldots, m    
\end{align*}
where we set $\hat{A} = A U $.
We then notice that $\mathbf{\hat{r}}^T \Lambda \mathbf{\hat{r}} =  \sum_i \lambda_i \hat{r}_i^2$.
By using the methods developed in \cite{cardelli2018robustness} it is straightforward to find coefficients of a linear under approximations $\alpha_i$ and $\beta_i$ such that: $\alpha_i + \beta_i \hat{r}_i \leq - \lambda_i \hat{r}_i^2$ for $i = 1,\ldots,M$.
Defining $\mathbf{\beta} = [\beta_1,\ldots,\beta_M]$, and $\hat{\alpha} = \sum_{i = 1}^M \alpha_i$ we then have that the solution to the following linear programming problem provides a valid lower bound to Problem \ref{prob:concave} and can be hence used to compute a lower bound to the latent variance:
\begin{align*}
    \min \left( \hat{\alpha} + \mathbf{\beta}^T \mathbf{\hat{r}} \right) \\
    \textrm{Subject to:} \quad &\hat{A}_{\hat{r}} \mathbf{\hat{r}} + A_x x \leq b  \\
    &\hat{r}_i^L \leq \hat{r}_i \leq \hat{r}_i^U \qquad  i = 1,\ldots, M    \\
    &x_i^L \leq x_i \leq x_i^U \qquad  i = 1,\ldots, m.    
\end{align*}

%\section{Branch and bound algorithm}
%\label{app:algorithm}
%\input{sections/appendix_algorithm.tex}

\section{RUNTIME ANALYSIS}
\label{app:runtime}
In this section of the Supplementary Material we empirically analyse the CPU time required for convergence of Algorithm \ref{alg:bnb_sketch} in the MNIST38 dataset. For the first $50$ test points and a $\gamma-$ball $T$ of dimensionality $d$, we calculated $\piSup{T}$ up to a pre-specified error tolerance $\epsilon$. We use $\gamma = 0.125$ and $\gamma = 0.25$, corresponding to up to $50 \%$ of the normalised input domain.
All runtimes analysed below were obtained on a MacBook Pro with a 2.5 Ghz Intel Core i7 processor and 16GB RAM running on macOS Mojave 10.14.6.

\begin{figure*}
	\centering
	{\hspace*{-.5cm}  \includegraphics[width = 0.4\textwidth]{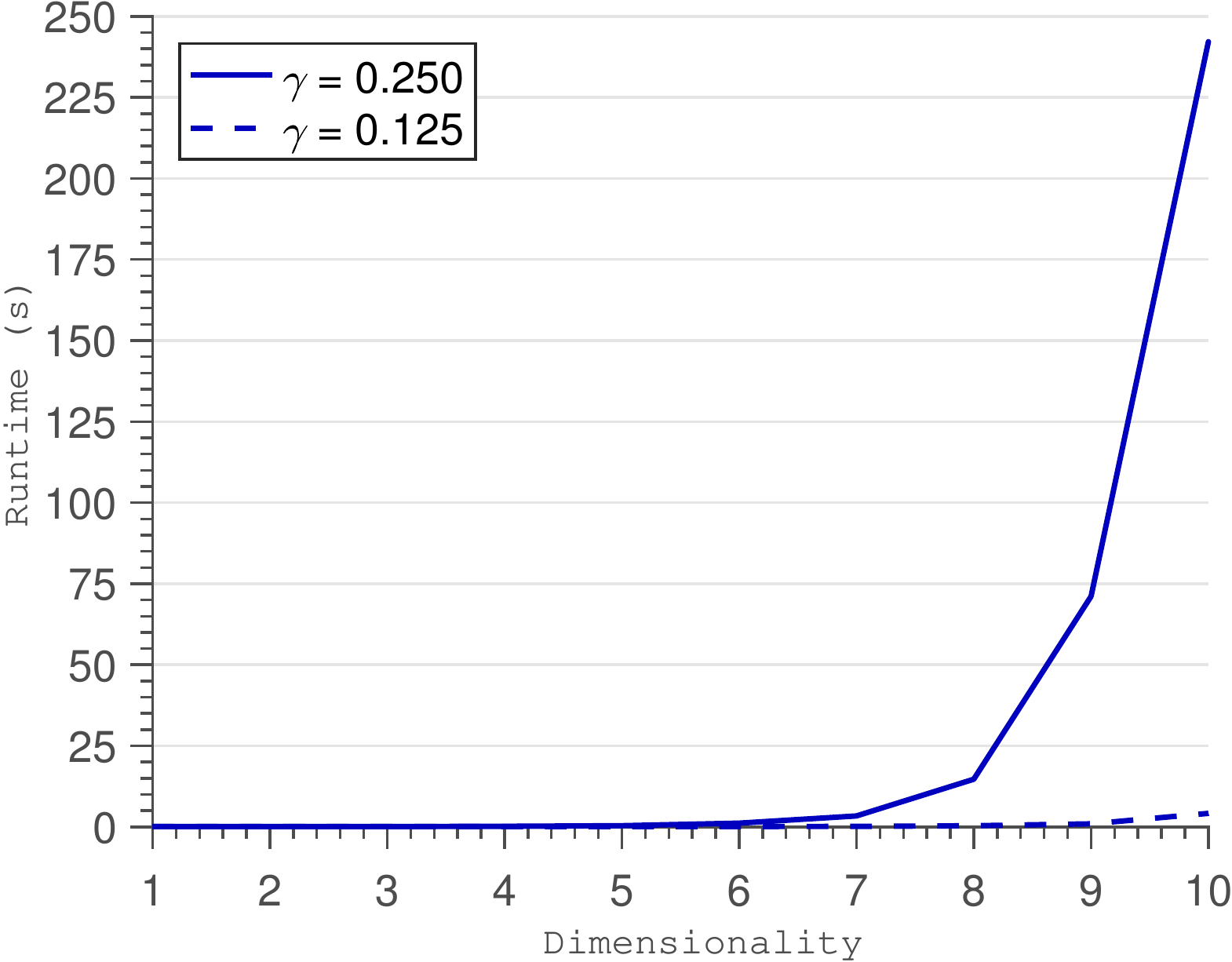}}
	{\hspace*{.5cm}  \includegraphics[width = 0.4\textwidth]{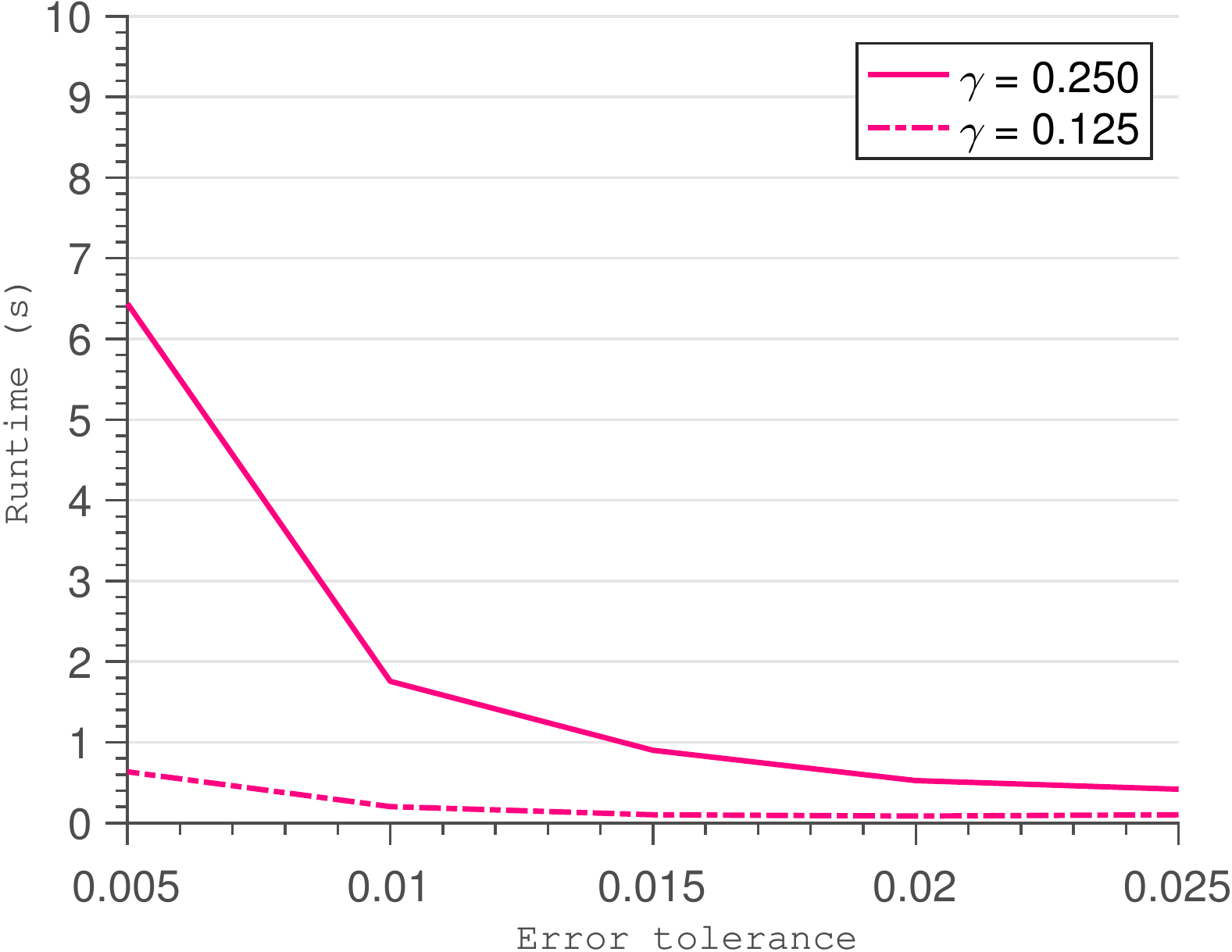}}
	\caption{Average runtimes of Algorithm \ref{alg:bnb_sketch} to calculate $\piSup{T}$ up to specified error tolerance $\epsilon$ among the first $50$ test points of MNIST38. \textbf{Left}: Average runtimes for increasing number of dimensions at $\epsilon = 0.025$. \textbf{Right}: Average runtimes for different values of $\epsilon$ with number of dimensions $d = 5$.}
	\label{fig:runtime}
\end{figure*}

\subsection{Runtime Depending on Dimension of Compact Subset.}
First, we analysed the effect of increasing $d$, by fixing $\epsilon = 0.025$ and increasing the number of pixels selected by SIFT to define $T$ from $1$ to $10$. The results are shown in terms of average runtime in Figure \ref{fig:runtime} on the left. For $\gamma = 0.25$, we can observe the exponential behaviour of the computational complexity in terms of number of dimensions, as the runtime quickly grows from below $5$ seconds to almost $250$ seconds beyond $7$ dimensions. However, for $\gamma = 0.125$ the exponential curve seems to be shifted further to the right, as still for $10$ dimensions Algorithm \ref{alg:bnb_sketch} terminates in only a few seconds. Given that for $\gamma = 0.125$, $T$ spans up to $25 \%$ of the input domain (on the selected pixels), we consider this quite fast.

\subsection{Runtime Depending on Error Tolerance}
Second, we analysed the effect of the error tolerance $\epsilon$, by calculating the bounds for each $\epsilon \in \{0.005, 0.01, 0.015, 0.02, 0.025\}$ with the number of pixels selected by SIFT (i.e. $d$) fixed to $5$. The results are shown in Figure \ref{fig:runtime} on the right. The behaviour seems to be roughly inversely exponential this time with lower error tolerance $\epsilon$ naturally demanding higher runtimes. In practice, one would seldom expect to require precision of $\epsilon < 0.01$ though, at which point Algorithm \ref{alg:bnb_sketch} still terminates in under $2$ seconds on average even for $\gamma = 0.25$. 

\section{EXPERIMENTAL SETTINGS}
\label{app:experiments}

\subsection{Datasets}

Our synthetic two-dimensional dataset contains 1,200 points, of which 50 \% belong to Class 1 and 50 \% belong to Class 2. The points were generated by shifting draws from a two-dimensional standard-normal random variable by $5$, either along the first dimension (Class 1) or along the second dimension (Class 2). Subsequently, we normalise the data by subtracting its mean and dividing by its standard deviation.

SPAM is a binary dataset that contains 4,601 samples, of which 60\% are benign. Each sample consists of 54 real-valued and three integer-valued features. However, identical or better prediction accuracies can be achieved with models involving only 11 of those 57 variables, among them e.g. the frequency of the word 'free' in the email, the share of \$ signs in its body, or the total number of capital letters, which is why we only use these 11 selected variables. We normalise the data by subtracting its mean and dividing by its standard deviation. 

MNIST38 contains 8,403 samples of images of handwritten digits, of which roughly 50 \% are 3s and 50 \% are 8s. Each sample consists of a $28 \times 28$ pixel image in gray scale (integer values between $0$ and $255$) which following convention, we normalise by dividing by $255$. For better scalability we then downsample to $14 \times 14$ pixels.

The subset of MNIST which we use to compile MNIST358 contains 5,715 samples of images of handwritten digits, of which roughly 36 \% are 3s, 31 \% are 5s and 34 \% are 8s. Each sample consists of a $28 \times 28$ pixel image in gray scale (integer values between $0$ and $255$) which following convention, we normalise by dividing by $255$. For better scalability we then downsample to $14 \times 14$ pixels.

\subsection{Experimental Settings}

For the binary experiments, we use 1,000 randomly selected points as a training set and 200 randomly selected points as a test set. For the multiclass experiments, scalability of GPs is even more of an issue so we just work with 500 randomly selected points as training set.

For the GP training of binary classification problems, we use the GPML package for Matlab. For the GP training of multiclass classification problems, we use the GPstuff package. 

For the safety verification experiments in Section \ref{subsec:safety}, we used a GPC model with a probit likelihood function and the Laplace approximation for the posterior. For the synthetic 2D data, the number of epochs (marginal likelihood evaluations) performed during hyper-parameter optimisation was restricted to 20. For the SPAM data, it was restricted to 40. Finally for the attacks on MNIST38 it was restricted to 10 and 20.

For the robustness experiments in Section \ref{subsec:robustness}, we give the specifications of training in the paper itself.

For the interpretability experiments in Section \ref{subsec:interpretability}, we use a multiclass GPC model with softmax link function and the Laplace approximation for the posterior. We limit the number of iterations performed during hyper-parameter optimisation to 10. 

The code for the GPFGS attacks as well as LIME was implemented by us in Matlab according to the original Python code provided by the authors.

%\subsection{GPFGS Attacks}
%GPFGS attacks have been developed in \cite{grosse2018limitations} as the GP analogue of the well-known fast gradient sign method (FGSM) attack for neural networks \cite{goodfellow2014explaining}. Given a perturbation budget $\gamma$, FGSM attacks a model at test point $x^*$ by perturbing each dimension by $\gamma$ in the direction of its gradient at $x^*$. Since for GPs, the integral in Eqn \ref{Eq:MultiClassClassification} makes calculating the gradient directly impossible, \cite{grosse2018limitations} propose to base the attack on the gradient of the latent mean function $\mu(\cdot)$ instead. The GPFGS attack is thus defined by
%\begin{equation}
%x_{adv} = x^* + \gamma \times \mathrm{sign}(\nabla \mu(x^*))
%\end{equation}

%\subsection{LIME}
%\LL{Not needed}
%\rev{Write a short section on LIME}

%The code for LIME was implemented by us in Matlab according to the original Python code provided by the authors.

%\section{Derivation of Interpretability metric $\mathbf{\Delta}^i_{\gamma}$}
\section{DETAILS ON INTERPRETABILITY METRIC}
\label{app:interpretability}

%\subsection{Derivation of Interpretability metric $\mathbf{\Delta}^i_{\gamma}$}

%The canonical way to understand which features are relevant to a GP classification model with an ARD kernel is to analyse the lengthscales of each feature. However, this has multiple drawbacks. Apart from only being feasible when an ARD kernel is used, it does only allow for a global analysis of feature relevance. Furthermore, it is directionless, i.e. while the lengthscale of a feature could tell us how relevant that feature is to the classifier, it can not tell us if an increase in its value is more likely to increase a class probability assigned by the model or to decrease it. 
%In order to shed light on the direction of the relevance of a feature, a popular method for any kind of model is LIME , which analyses the impact of each feature by building a local linear approximation model imitating the behaviour of the model to be analysed. While this allows for both a local and a directed analysis of feature relevance, a potential drawback lies in the question whether a good linear approximation to the possibly even locally highly-nonlinear model can be achieved.

Below, we briefly derive our metric for interpretability analysis $\mathbf{\Delta}^i_{\gamma}$, which by using our bounds does not rely on local linearity , in a bit more detail.
%and compare the outcomes for three different datasets.

%\subsection{A metric for interpretability analysis}

For a testpoint $x^*$ and dimension $i$, we define $T^i_{\gamma}(x^*) = [x^*, x^* + \gamma*e_i]$ like in the main paper. To analyse the impact of changes in dimension $i$, we propose to analyse how much the maximum of the assigned class probabilities can differ from the initial class probability $\pi(x^*)$ over such a one-sided interval compared to  how much the minimum differs from that initial probability. In other words, we calculate 
\begin{align}
    \Delta^i_{\gamma}(x^*) = & \left( \piSup{T^i_{\gamma}(x^*)}  -\pi(x^*)\right) \\
    & - \left(\pi(x^*)-\piInf{T^i_{\gamma}(x^*)}\right). \label{eq:interpret}
\end{align}
If increasing the value of dimension $i$ makes the model favor assigning lower class probabilities, we would expect this value to be negative and vice versa. To make it more robust, we center the analysis by calculating the proposed metric 
\begin{align}
    \mathbf{\Delta^i_{\gamma}}(x^*) = & \Delta^i_{\gamma}(x^*) - \Delta^i_{-\gamma}(x^*) \\
    = &\left(\piSup{T^i_{\gamma}(x^*)} -\piSup{T^i_{-\gamma}(x^*)}\right)\\
    & + \left(\piInf{T^i_{\gamma}(x^*)}-\piInf{T^i_{-\gamma}(x^*)}\right). \label{eq:interpretcenter}
\end{align}
Finally, if instead of a local analysis a global analysis is desired, we suggest following LIME's approach in aggregating local insights to a global insight by averaging over a selection of test points $M$
\begin{equation}
    \mathbf{\Delta^i_{\gamma}} = \frac{1}{M} \sum_{j = 1}^{M} \mathbf{\Delta^i_{\gamma}}(x^j).
\end{equation}
Ideally, $M$ contains all test points; however, if for computational reasons a subselection is to be made, the SP algorithm in \cite{ribeiro2016should} could be used.

\end{document}